\documentclass[12pt]{article}
\usepackage{fullpage}

\usepackage{amsmath,amsthm,amsfonts,amssymb}
\usepackage{latexsym}
\usepackage{bm}
\usepackage{hyperref}
\usepackage[round]{natbib}
\usepackage{framed}
\usepackage{float}
\usepackage[shortlabels]{enumitem}

\usepackage{tikz}
\usetikzlibrary{automata,positioning,fit,backgrounds}

\usepackage{amsthm,amsmath,amsfonts,amssymb}
\usepackage[capitalise]{cleveref}
\usepackage{xcolor}

\newcommand{\ignore}[1]{}

\theoremstyle{plain}
\newtheorem{theorem}{Theorem}
\newtheorem{lemma}[theorem]{Lemma}

\newtheorem{proposition}[theorem]{Proposition}

\newtheorem*{theorem*}{Theorem}
\newtheorem*{lemma*}{Lemma}
\newtheorem*{corollary*}{Corollary}
\newtheorem*{proposition*}{Proposition}
\newtheorem*{claim*}{Claim}
\newtheorem*{fact*}{Fact}

\theoremstyle{definition}
\newtheorem{definition}[theorem]{Definition}
\newtheorem*{definition*}{Definition}

\newtheorem*{remark*}{Remark}

\newtheorem*{example*}{Example}

\theoremstyle{plain}
\newtheorem*{theoremaux}{\theoremauxref}
\gdef\theoremauxref{1}

\newenvironment{repthm}[2][]{%
  \def\theoremauxref{\cref{#2}}
  \begin{theoremaux}[#1]
}{%
  \end{theoremaux}
}

\newcommand{\wt}[1]{\smash{\widetilde{#1}}}
\renewcommand{\O}{O}

\newcommand{\tO}{\wt{\O}}
\newcommand{\E}{\mathbf{E}}

\newcommand{\ind}[1]{1\!\!1_{#1}}
\newcommand{\non}{\nonumber}

\newcommand{\set}[1]{\{#1\}}
\newcommand{\abs}[1]{|#1|}
\newcommand{\norm}[1]{\|#1\|}
\newcommand{\lr}[1]{\left(#1\right)}
\newcommand{\lrbig}[1]{\big(#1\big)}

\newcommand{\eps}{\epsilon}
\newcommand{\sig}{\sigma}
\newcommand{\del}{\delta}
\newcommand{\Del}{\Delta}

\newcommand{\half}{\frac{1}{2}}

\hypersetup{colorlinks,
            linkcolor=blue,
            citecolor=blue,
            urlcolor=magenta,
            linktocpage,
            plainpages=false}

\newcommand{\eq}{~=~}
\renewcommand{\leq}{~\le~}
\renewcommand{\geq}{~\ge~}

\newcommand{\floor}[1]{\left\lfloor #1 \right\rfloor}
\newcommand{\gam}{\gamma}

\newcommand{\F}{\mathcal{F}}

\newcommand{\cut}{\mathrm{cut}}
\newcommand{\anc}{\rho^*}
\newcommand{\depth}{d}
\newcommand{\width}{w}
\newcommand{\clip}{\mathrm{clip}}

\floatstyle{plain}
\newfloat{algorithm}{h}{lop}
\floatname{algorithm}{Algorithm}

\newlength{\algindent}
\title{Chasing Ghosts: Competing with Stateful Policies \\[0.4cm]}

\date{}

\author{%
Uriel Feige\\
Weizmann Institute and MSR\\
\texttt{uriel.feige@weizmann.ac.il}
\and
Tomer Koren\\
Technion and MSR\\
\texttt{tomerk@technion.ac.il}
\and
Moshe Tennenholtz\\
MSR and Technion \\
\texttt{moshet@microsoft.com}
}

\begin{document} 
\maketitle

\thispagestyle{empty}
\begin{abstract}
We consider sequential decision making in a setting where regret is measured with respect to a set of {\em stateful} reference policies, and feedback is limited to observing the rewards of the actions performed (the so called ``bandit" setting). If either the reference policies are stateless rather than stateful, or the feedback includes the rewards of all actions (the so called ``expert" setting), previous work shows that the optimal regret grows like $\Theta(\sqrt{T})$ in terms of the number of decision rounds $T$.

The difficulty in our setting is that the decision maker unavoidably loses track of the internal states of the reference policies, and thus cannot reliably attribute rewards observed in a certain round to any of the reference policies. 
In fact, in this setting it is impossible for the algorithm to estimate which policy gives the highest (or even approximately highest) total reward. 
Nevertheless, we design an algorithm that achieves expected regret that is sublinear in $T$, of the form $O( T/\log^{1/4}{T} )$. Our algorithm is based on a certain {\em local repetition lemma} that may be of independent interest. We also show that no algorithm can guarantee expected regret better than
$O( T/\log^{3/2} T )$.
\end{abstract}

\newpage
\setcounter{page}{1}

\section{Introduction}

A player is faced with a sequential decision making task, continuing for $T$ rounds. There is a finite set $[n] = \{1 , \ldots, n\}$ of actions available in every round. In every round, based on all information observed in previous rounds, the player may choose an action $i \in [n]$, and consequently receives some reward $r \in [0,1]$ on that particular round. The total reward of the player is the sum of rewards accumulated in all rounds. There are various policies suggested to the player as to how to choose the sequence of actions in a way that would lead to high total reward. Examples of policies can be to play action~2 in all rounds, to play action~2 in odd rounds and action~3 in even rounds, or to start with action~1, play the current action repeatedly in every round until the first round in which it gives payoff less than $1/2$, then switch to the next action in cyclic order, and so on. The number of given policies is denoted by $k$. A-priori the player does not know which is the better policy. An algorithm of the player is simply a new policy that may be based on the available given policies. For example, the algorithm may be to follow policy number~5 in the first $T/2$ rounds, and play action~3 in the remaining rounds. The regret of the algorithm of the player is the difference between the total payoff of the best given policy to that of the player's algorithm. Our goal is to design an algorithm for the player that has as small regret as possible.

There are many different variations on the above setting, and some have been extensively studied in the past, with two of the most common variations referred to as expert algorithms and bandit algorithms~\citep{cesa1997use,freund1997decision,auer2002nonstochastic}. 
In this work we study a natural variation that apparently did not receive much attention in the past.
We present this variation in its simplest form in \cref{sec:model}, and defer discussion of extensions to \cref{sec:extensions}.


\subsection{The Stateful Policies Model} \label{sec:model}

We view the sequential decision making problem as a repeated game between a player and an adversary.
Before the game begins, the adversary determines a sequence%
\footnote{We use the notation $a_{s:t}$ as a shorthand for the sequence $(a_{s},\ldots,a_{t})$.}
of reward functions $r_{1:T} = (r_{1},\ldots,r_{T})$, where each function assigns each of the actions in $[n]$ with a reward value in the interval $[0,1]$.
We refer to such adversary as \emph{oblivious}, since the functions $r_{1:T}$ cannot change as a result of the player's actions (as they are chosen ahead of time).
On each round $t$, the player must choose, possibly at random, an action $X_{t} \in [n]$. He then receives the reward $r_{t}(X_{t})$ associated with that action, and his feedback on that round consists of this reward only; this is traditionally called {\em bandit feedback}.

The player is given as input a set $\Pi$ of $k > 1$ \emph{policies}, which are referred to as the {\em reference policies}.
Each policy $\pi \in \Pi$ is a deterministic function that maps the sequence of all previously observed rewards into an action to be played next.
For a policy $\pi$, we use the notation $x_{t}^{\pi}$ to denote the action played by $\pi$ on round $t$, had it been followed from the beginning of the game (note that $x_{t}^{\pi}$ has a deterministic value).
The player's goal is to minimize his (expected) \emph{regret} measured with respect to the set of reference policies $\Pi$, defined by
\begin{align*}
	\mathrm{Regret}_{T} \eq
	\max_{\pi \in \Pi} \sum_{t=1}^{T} r_{t}(x_{t}^{\pi})
	- \E\left[ \sum_{t=1}^{T} r_{t} \big( X_{t} \big) \right] ~.
\end{align*}
We say that the player's regret is non-trivial if it grows sublinearly with $T$, namely if $\mathrm{Regret}_{T} = o(T)$.

While regret measures the performance of a specific algorithm on a particular sequence of reward functions, we are typically interested in understanding the intrinsic difficulty of the learning problem.
This difficulty is captured by the game-theoretic notion of \emph{minimax regret}, which intuitively is the expected regret of an optimal algorithm when playing against an optimal adversary. 
Formally, the minimax regret is defined as the infimum over all player algorithms, of the supremum over all reward sequences, of the expected regret.

In this paper we consider a type of reference policies that we refer to as \emph{stateful policies}, which we define next (see also \cref{fig:policy} for an illustration of this concept).

\begin{definition}[stateful policy]
A \emph{stateful policy} $\pi = (s_{0}^{\pi},f^{\pi},g^{\pi})$ over $n$ actions and $S$ states is a finite state machine with state space $[S] = \set{1,2,\ldots,S}$, characterized by three parameters:
\begin{enumerate}[(i)]
\item
the \emph{initial state} of the policy $s_{0}^{\pi} \in [S]$, which is used to initialize the policy before the first round;
\item
the {\em action function} $f^{\pi}:[S] \mapsto [n]$, describing which action to take in a given round, depending on the state the policy is in;
\item
the {\em state transition function} $g^{\pi}:[S] \times [0,1] \mapsto [S]$, which given the current state and the observed reward of the action played in the current round, determines to which state to move for the next round.
\end{enumerate}
\end{definition}

\begin{figure}[h]
\centering

\begin{tikzpicture}[>=latex,text height=1.5ex,text depth=0.25ex]

\tikzstyle{state} = [
	circle,
	minimum size=1.1cm,
	thick,
	draw=black,
	fill=white
	]

\tikzstyle{background} = [
	rectangle,
	rounded corners=5mm,
	inner xsep=0.75cm,
	inner ysep=0.2cm,
	fill=gray!10
	]
  
\matrix[row sep=0.5cm,column sep=1.2cm] {
	\node (s_1) [state,double] {$s_1$}; 
	&
	&
	\node (a_1) {$\; 1$}; 
	\\
	
	&
	\node (s_2) [state] {$s_2$}; 
	&
	\node (a_2) {$\; 2$}; 
	\\

    	\node (s_3) [state] {$s_3$};
	&
	&
	\node (a_3) {$\; 1$}; 
	\\
};
    
\path[->,thick]
	(s_1)	edge[bend left=30] 
		node[right=4pt,near end]{$\scriptstyle[0,\frac{1}{2})$} 	(s_2)
	(s_1) 	edge[bend left=30] 
		node[right,very near end]{$\scriptstyle[\frac{1}{2},1)$}	(s_3)
	(s_2) 	edge[out=210,in=160,loop] 
		node[above=24pt,right]{$\scriptstyle[0,\frac{1}{2})$} 	(s_2)
	(s_2) 	edge[bend left=30] 
		node[right,near start]{$\scriptstyle[\frac{1}{2},1)$} 		(s_3)
	(s_3) 	edge[bend left=30] 
		node[left,very near end]{$\scriptstyle[0,1]$}	 		(s_1)
	;

\path[->,thick,dashed]
	(s_1) 	edge[bend left=10]		(a_1)
	(s_2) 	edge				(a_2)
	(s_3) 	edge[bend right=10]	(a_3)
	;

\begin{pgfonlayer}{background}
	\node [
		background,
                    fit=(s_1) (s_2) (s_3),
                    ] {};
\end{pgfonlayer}

\end{tikzpicture}
%
\hspace{0.05\linewidth}
%
\begin{tikzpicture}[>=latex,text height=1.5ex,text depth=0.25ex]

\tikzstyle{state} = [
	circle,
	minimum size=1.1cm,
	thick,
	draw=black,
	fill=white
	]

\tikzstyle{background} = [
	rectangle,
	rounded corners=5mm,
	inner xsep=0.3cm,
	inner ysep=0.2cm,
	fill=gray!10
	]
  
\matrix[row sep=0.5cm,column sep=1.2cm] {
	\node (e_1) {}; 
	&&
	\node (s_1) [state,double] {$s_1$}; 
	&
	\node (a_1) {$\; 1$}; 
	\\
	
	\node (e_2) [state] {$s_{i}$}; 
	&&
	\node (s_2) [state] {$s_2$}; 
	&
	\node (a_2) {$\; 2$}; 
	\\

	\node (e_3) {}; 
	&&
    	\node (s_3) [state] {$s_3$};
	&
	\node (a_3) {$\; 3$}; 
	\\
};
    
\path[->,thick]
	(e_2) edge[-] (e_1.center)
	(e_1.center)	edge	node[above=3pt]{$\scriptstyle[\frac{1}{3},\frac{2}{3}]$} 	
			(s_1)
	(e_2)	edge	node[above=3pt]{$\scriptstyle[0,\frac{1}{6}) \,\cup\, (\frac{5}{6},1]$} 		
			(s_2)
	(e_2) edge[-] (e_3.center)
	(e_3.center)	edge	node[above=3pt]{$\scriptstyle[\frac{1}{6},\frac{1}{3}) \,\cup\, (\frac{2}{3},\frac{5}{6}]$} 
			(s_3)
	;

\path[->,thick,dashed]
	(s_1) 	edge	(a_1)
	(s_2) 	edge	(a_2)
	(s_3) 	edge	(a_3)
	;

\begin{pgfonlayer}{background}
	\node [
		background,
                    fit=(e_2) (s_1) (s_2) (s_3),
                    ] {};
\end{pgfonlayer}

\end{tikzpicture}
\caption{(Left) A stateful policy over $n=2$ actions with $S=3$ states, with $s_{1}$ being the initial state. The labels on the edges between states indicate the set of rewards that trigger the corresponding transition (which is the role of the function $g^{\pi}$). The dashed arrows depict the function $f^{\pi}$, that assigns each state with an action.
(Right) A reactive policy over $n=3$ actions, considered in the example of \cref{sec:model}. The state $s_{i}$ is a placeholder that stands for each of $s_{1},s_{2},s_{3}$ and shows the outgoing transitions that are common to all 3 states.
} \label{fig:policy}
\end{figure}
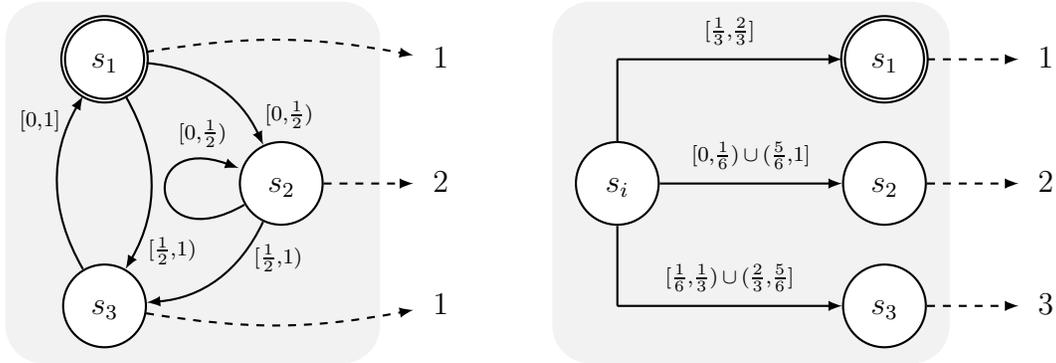

The action $x_{t}^{\pi}$ played by a stateful policy $\pi$ on round $t$ (had $\pi$ been followed from the beginning of time) can be computed recursively, starting from the given initial state $s_{0}^{\pi}$, according to
$$
\forall ~ t \in [T]~, \qquad
\begin{aligned}
	x_{t}^{\pi} &\eq f^{\pi}(s_{t-1}^{\pi}) ~, \\
	s_{t}^{\pi} &\eq g^{\pi} (s_{t-1}^{\pi},r_{t}(x_{t}^{\pi})) ~.
\end{aligned}
$$
Here, $s_{t}^{\pi}$ represents the state $\pi$ reaches at the end of round $t$.

In our setting, we assume that the player is given as input a reference set $\Pi$ of $k > 1$ stateful policies, each over at most $S$ states.
The player may base his decisions on the description of the $k$ reference policies (in particular, the policies can serve as subroutines by his algorithm).
Without loss of generality, we shall assume that each policy in $\Pi$ has exactly $S$ states.
Also, for simplicity we assume that policies are \emph{deterministic} (involve no randomization) and {\em time-independent}: the functions $f^{\pi}$ and $g^{\pi}$ do not depend on the round number; see \cref{sec:extensions} for extensions of our results to randomized and to time-dependent policies.

\begin{example*}
We present a detailed example to illustrate the model.
Suppose that our player, a driver, faces a daily commute problem, that repeats itself for a very large number $T$ of days. There are three possible routes that he can take, and an action is a choice of route (hence $n=3$). Each of the three routes can be better than the others on any given day. The reward of the player on a given route in a given day is some number in the range $[0,1]$ that summarizes his satisfaction level with the route he took (taking into account the time of travel, road conditions, courtesy of other drivers, and so on). The driver learns this reward only after taking the route, and does not know what the reward would have been had he taken a different route. We further assume that the effect of a single driver on traffic experience is negligible: the presence of the driver on a particular route on a given day has no effect of the quality (satisfaction level) of that or any other route on future days.

The driver is told that there is a useful policy for choosing the route in a given day, based only on the reward of the previous day. This policy has three states (hence $S = 3$). The action function $f$ is simply the identity function (in state $i$ take route $i$). The state transition function $g$ is independent of the current state, and depends only on the reward received. If $x$ denotes the reward received in the current day, then the next state is as follows:
$g(x) = 1$ for $|x - \frac{1}{2}| \le \frac{1}{6}$, $g(x) = 2$ for $|x - \frac{1}{2}| > \frac{1}{3}$, and $g(x) = 3$ otherwise.  The only part not specified by the policy is the initial state $s_0$ (which route to take on the first day). Hence effectively there are three reference policies, different only on their initial state, and thus $k=3$.

The beauty of the policy, so the player is told, is that if he gets the initial state right and from then on follows the policy blindly, his overall satisfaction is guaranteed.  Not knowing which is the better reference policy, does the player have a strategy that guarantees sublinear regret (in $T$) against the best of the three reference policies? If so, how low can this regret be guaranteed to be?
\end{example*}

The kind of policies considered in our example above is perhaps the weakest type of a stateful policy, one that we refer to as a \emph{reactive policy}.

\begin{definition}[reactive policy] \label{def:reactive}
A \emph{reactive policy} $\pi$ over $n$ actions (with 1-lookback) is specified by an initial action $x_{1}^{\pi} \in [n]$ to be played in the first round of the game, and by a function $\pi : [0,1] \mapsto [n]$ that maps the observed reward of the action played in the current round to an action to be played on the next round.
\end{definition}

A reactive policy simply reacts to the last reward it receives as feedback and translates it into an action to be played on the next round.
A reactive policy can be seen as a special type of a stateful policy with $S = n$ states, if we identify each of the sets $\pi^{-1}(i) \subseteq [0,1]$ with a unique state $i \in [n]$.
In this view, the action function $f^{\pi}$ is simply the identity function, and the state transition function $g^{\pi}$ is independent of the current state (and maps a reward $r$ to the state $i$ if $r \in \pi^{-1}(i)$).
See also \cref{fig:policy} for a visual description of the reactive policies used in our example.

\subsection{Main Results}
\label{sec:results}

We now state our main results, which are upper and lower bounds on the expected regret in the stateful policies model.

\begin{theorem}
\label{cor:positive}
For any given $k,S \ge 1$, there is an algorithm for the player that guarantees sublinear expected regret with respect to any reference set $\Pi$ of $k$ stateful policies over $S$ states.
Specifically, for any set $\Pi$ and any oblivious sequence of reward functions, \cref{alg:policies} given in \cref{sec:policy-upper} achieves an 
$$
	O\lr{ \sqrt{k S} \cdot \frac{T \log\log{T} }{\log^{1/4} {T} } }
$$
upper bound %
over the expected regret with respect to $\Pi$.
\end{theorem}

Though the regret achieved in \cref{cor:positive} is sublinear, it is only slightly so. Unfortunately, this is unavoidable.


\begin{theorem}
\label{cor:negative}
No player algorithm can guarantee expected regret better than $O( T/\log^{3/2}{T} )$ with respect to any set of $k=3$ reference policies over $S=3$ states and $n=3$ actions, not even if the reference policies are all reactive (as in \cref{def:reactive}).
Moreover, this negative result holds in the commute example given in \cref{sec:model}.
\end{theorem}


For proving the above bounds, it will be convenient for us to first obtain upper and lower regret bounds in a simplified model we call \emph{the hidden bandit}. 
This model precisely captures the main difficulties associated with the stateful policies setting, and may be of independent interest.
Our results in the hidden bandit setting will be stated after we establish the required definitions in \cref{sec:bandit}.



\subsection{Discussion and Related Work}

A unifying paradigm for virtually all previous sequential optimization algorithms, whether in the expert or bandit setting, is the following. As the rounds progress, the algorithm ``learns" which arm had the better past performance (in the expert setting the algorithm observes all arms, in the bandit setting the algorithm uses an ``exploration and exploitation" procedure), and then plays this arm (either deterministically or with high probability). This paradigm is not suitable for our stateful policies model (with bandit feedback), as there is no way by which the algorithm can learn the identity of the best reference policy, even if this reference policy gives reward~1 in every round and all other reference policies give reward~0 in every round. This difficulty stems from the fact that reference policies might differ only in their initial state, and their identity is lost because the player cannot track the state evolution of policies, due to the bandit nature of the feedback. (This aspect will become more evident in the proof of \cref{cor:negative}.)

Several variants of our model have been extensively studied in the past.
However, to the best of our knowledge, our results constitute the first known example of a learning problem where the minimax regret rate is of the form $\Theta(T/\mathrm{polylog}(T))$.
For this reason, we believe that the problem we consider is substantially different from previously studied, seemingly related sequential decision problems.

The full-feedback analog of our setting is widely known to be captured by the so-called ``experts'' framework, and has been studied under the name of ``simulatable experts'' \citep{cesa2006games}.
Basically, when the player observes the rewards of all actions he is able to ``simulate'' each of his contending policies and keep track of their cumulative rewards. 
Hence, we can treat each policy as an independent expert and use standard online learning techniques (such as the weighted majority algorithm) to obtain $O(\sqrt{T})$ regret in this setting.
Consequently, we exhibit an \emph{exponential gap} between the minimax regret rates of the full-feedback and bandit-feedback variants of the problem%
\footnote{We say that the gap between the achievable rates is exponential, since the average (per-round) regret decays like $1/\mathrm{polylog}(T)$ in the bandit case, while in the full-information case it decays like $1/\sqrt{T}$.}. 
As far as we know, this is the first evidence of such gap to date:
the only previously known gap is in the case of the multi-armed bandit problem with switching costs, where the minimax regret rates are $\Theta(\sqrt{T})$ and $\wt{\Theta}(T^{2/3})$ in the full-feedback and bandit-feedback versions, respectively \citep{audibert2009minimax,dekel2013bandits}.

Among models with bandit feedback, the one most closely related to ours is perhaps the setting of the EXP4 algorithm \citep{auer2002nonstochastic}, which is a variation on the standard multi-armed bandit problem.
In this setting, on each round of the game, before committing to a single action and observing its reward the player is provided with the advice of a fixed set of ``experts'' on which arm to choose.
The player's goal is to perform as well as the best expert in the set, and his regret is computed with respect to that expert. 
\cite{auer2002nonstochastic} suggest the EXP4 algorithm for this setup and proves that it achieves an optimal $O(\sqrt{T})$ bound over the regret.
The crucial difference between this setting and ours is in the fact that the advice of an expert is assumed to be available at all times, whereas the advice of a stateful policy becomes unavailable once the player deviates from it.
In other words, while our policies are simple algorithms that observe bandit feedback, we think of their experts as ``oracles'' whose observation is not limited to the player's rewards.

Our setting might seem reminiscent of (online) reinforcement learning models, and in particular, of online Markov Decision Processes (MDPs) \citep{even2009online,neu2010online}. 
In these models, there is typically a finite number of states, and the player's actions on each round cause him to transition from one state to another.
As a consequence, the reward of the player on each round is determined not only based on his action on that round, but also as a function of his actions in previous rounds.
In contrast, in our setting the environment is oblivious and thus determines the reward based solely on the player's action on the current round.
Furthermore, 
in an MDP the state is \emph{of the environment} and the player's actions inevitably causes this state to change from round to round; in our model, the state is \emph{owned by the player} (more precisely, by his contending policies) and he may freely transition himself to an arbitrary state (of any one of the policies) at any given moment, or even choose not to be in any of the states.

More generally, the settings considered in the related works of \cite{merhav1998universal}, \cite{farias2006combining} and \cite{arora2012online} (among others), that deal with stateful and reactive environments in an online decision making framework, are also substantially different than ours. 
As is the case with the reinforcement learning literature, the focus of these works is the adaptiveness of the adversary and not of the player's reference policies. 

Finally, we remark that our definition of a stateful policy is not new and similar notions have been considered in the past.
Most notably, the work of \cite{feder1992universal} in the related context of binary sequence prediction considers a similar concept which they call ``FS predictor'', and studies the prediction power of the class of all such predictors with at most $S$ states. However, our goal is entirely different than theirs: while they are concerned with the prediction power of the class of all such predictors with at most $S$ states, we aim to understand the difficulty of learning a small set of these concepts (with bandit feedback).

\subsection{The Hidden Bandit Problem}
\label{sec:bandit}

In this section we present a setting that we shall refer to as {\em the hidden bandit problem}, which captures the main difficulties associated with the stateful policies model. It will be convenient for us to first obtain results in the hidden bandit model, and then translate them to the stateful policies model.

\begin{figure}
\begin{center}
\begin{tikzpicture}[line width=1pt]
\matrix[row sep=0.5cm,column sep=1.2cm] {
	\node (0) {reference}; 
	&
	\node (1) {decoy}; 
	\\
	\node (0) [state] {$0$}; 
	&
	\node (1) [state] {$1$}; 
	\\
};
\draw[
	>=latex,
	auto=right,
	loop right/.style={out=-30,in=30,loop},
	loop left/.style={out=210,in=150,loop},
	every loop,
	]
(0)	edge[loop left]		node[left]	{$0$}	(0)
(1)	edge[loop right]		node		{$1-p$}	(1)
(0)	edge[bend right=30] 	node		{$1$}	(1)
(1)	edge[bend right=30] 	node		{$p$}	(0);
\end{tikzpicture}
\end{center}
\caption{The dynamics of the \textsf{switch} action in the hidden bandit model can be viewed as a two-state Markov chain, where state $0$ stands for the reference arm, and state $1$ for the decoy arm. The arrow labels denote the corresponding arm transition probabilities. 
}
\label{fig:switch}
\end{figure}
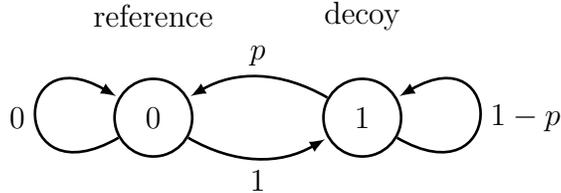

To motivate the hidden bandit problem, let us discriminate between two different modes a player in the stateful policies model may be in, at any given round: the ``good mode'', in which the algorithm is following the best reference policy in its correct state, and the ``bad mode'' in which the algorithm is doing something else (i.e., following other reference policy or executing a sequence of actions that do not correspond to any reference policy).
Inevitably, the player is not aware of his current mode and is unable to switch  between the modes deterministically.
However, if at some point in time the player is told that he is in the ``good mode'', then from that point onwards he can replicate the actions of the best policy by observing its rewards and emulating its state transitions, and remain in the same mode.

Roughly, the hidden bandit problem can be described as a multi-armed bandit problem with two arms, the \emph{reference arm} and the \emph{decoy arm}, that 
correspond to the ``good mode'' and the ``bad mode" in the stateful policies model, respectively. 
Unlike standard multi-armed bandit problems, a key aspect of this problem is that in any given round the player does not know which of the arms he is currently pulling. 
Accordingly, the player is not able to select which arm to pull on each round; rather, he can only choose whether to \emph{stay} on the current arm or to \emph{switch} to the other arm with some probability.
These aspects capture the difficulties in the stateful policies model, in which once the player leaves a certain policy, attempting to return to that policy involves guessing correctly the policy's internal state, an aspect that a player is not sure of.

\paragraph{The model.}

We now turn to the formal description of the hidden bandit model.
There are two parameters associated with the hidden bandit model. One is $T$, the number of rounds, and the other is $p$, a parameter in the range $0 < p < 1$.
There are two arms, arm~0 and arm~1, that will be referred to as the {\em reference arm} and the {\em decoy arm}, respectively. 
At each round, the player has only two possible actions available:
\begin{itemize}
\item
\textsf{stay}: stays on the same arm on which the player entered the round;
\item
\textsf{switch}: switches to arm $1$ if the player entered the round on arm~0; otherwise, switches to arm $0$ with probability $p$, and stays on arm $1$ with probability $1-p$.
\end{itemize}
The dynamics of the \textsf{switch} action can be seen as a two-state Markov chain, illustrated in \cref{fig:switch}.
Initially, prior to round~1, the player is placed on one of the arms at random, being on arm~0 with probability $\frac{p}{1 + p}$ and on arm 1 with probability $\frac{1}{1+p}$. This initial probability distribution is the stationary distribution with respect to the randomized \textsf{switch} action defined above.
Hence, any sequence of actions (either \textsf{stay} or \textsf{switch}) of the player gives rise to a sequence of random variables $X_{1:T}$, where $X_{t} \in \set{0,1}$ indicates which arm is pulled by the player on round $t$.
Even though at each round the player is pulling some arm, {\em the player cannot observe on which arm he is playing}.
In other words, the sequence $X_{1:T}$ is \emph{not observable} by the player.

On each round $t=1,\ldots,T$, the adversary assigns a reward to each arm.
We let $r_t(i) \in [0,1]$ denote the reward of arm $i$ on round $t$. 
The rewards of the reference arm are set by the adversary in an {\em oblivious} way, before the game begins. 
The rewards of the decoy arm are set by the adversary in an {\em adaptive} way as the game progresses: at every round $t$, the reward of arm~1 can be based on the entire history of the game up to round $t$. 
The feedback to the player on round $t$ in which arm $X_{t}$ is played consists only of the reward $r_{t}(X_{t})$, and the player does not get to observe the reward of the other arm on that round.

The goal of the player is to minimize his expected \emph{regret}, which is computed only with respect to the total reward of the \emph{reference arm}, namely
\begin{align*}
	\mathrm{Regret}_{T} \eq
	\sum_{t=1}^{T} r_{t}(0) - \E\left[ \sum_{t=1}^{T} r_{t}(X_{t}) \right] ~,
\end{align*}
where the expectation on the right-hand side is taken with respect to the randomization of the \textsf{switch} actions, as well as to the internal random bits used by the player.

This completes the description of the hidden bandit problem.

\paragraph{Remark.}


The fact that the adversary can set the rewards on the decoy arm in an adaptive manner will allow us to simulate any execution in the stateful policies model by an execution in the hidden bandit model. Consequently, all positive results (algorithms with low regret) that we shall prove in the hidden bandit model will transfer easily to the stateful policies model (basically, by setting $p = 1/(Sk)$, where $k$ is the number of reference policies and $S$ is the maximum number of states that a policy might have).
On the other hand, 
it might not be true that negative results in the hidden bandit model transfer to the stateful policies model. 
Nevertheless, our negative results for the hidden bandit model will be obtained with an \emph{oblivious adversary} (which is oblivious not only on the reference arm but also on the decoy arm), and consequently will transfer to the stateful policies model.

\paragraph{Results.}

We now present our results for the hidden bandit problem, that we later show how to translate into the corresponding upper and lower bounds in the stateful policies model.

\begin{theorem} \label{thm:positive}
For any given $0 < p < 1$, there is an algorithm for the player in the hidden bandit setting that guarantees sublinear expected regret (in $T$). 
Specifically, \cref{alg:onearm2} presented in \cref{sec:bandit-upper} achieves an expected regret of
$$
	O \lr{ \frac{1}{\sqrt{p}} \cdot \frac{T \log\log{T}}{\log^{1/4}{T}} } ~.
$$
over any sequence of reward functions.
\end{theorem}

\begin{theorem}
\label{thm:negative}
For $p = \frac{1}{2}$, no algorithm for the player in the hidden bandit setting guarantees expected regret better than $O ( T/\log^{3/2}{T} )$, not even if the adversary uses an oblivious strategy on both arms.
\end{theorem}

There is a gap between the upper bound of \cref{thm:positive} and the lower bound of \cref{thm:negative}, that translates into a gap between our main upper and lower bounds of \cref{cor:positive,cor:negative}.
In some natural special cases, we are able to close this gap.
We say that an adversary is {\em consistent} if there is a fixed offset $0 < \Del \le 1$ such that in every round $t$, $r_t(0) - r_t(1) = \Del$. Say that the player's algorithm is {\em semi-Markovian} if the choice of action taken at any given round depends only on the sequence of rewards obtained since the last switch action. (See exact definitions in \cref{sec:semimarkov}.)

\begin{theorem}
\label{thm:semimarkov}
In the hidden bandit setting, if the player's algorithm is required to be semi-Markovian and the adversary is required to be consistent, then there is an algorithm achieving expected regret $O(T/\log T)$, and this is best possible up to constants (that may also depend on $p$).
\end{theorem}

We remark that we actually prove a slightly stronger statement than that of \cref{thm:semimarkov}: for the positive results a {\em Markovian} algorithm suffices, for the negative results a {\em consistent} adversary suffices. See \cref{sec:semimarkov} for more details.

\subsection{Our Techniques and Additional Related Work}

Our algorithm in the proofs of \cref{thm:positive} and \cref{cor:positive} is based on a principle that to the best of our knowledge has not been used previously in sequential optimization settings. This is the {\em local repetition lemma} which will be explained informally here, and addressed formally in \cref{sec:localrepeat} (see \cref{lem:localrepetitive}).

In the hidden bandit setting, suppose first that the sequence of rewards that the adversary places on the reference arm is {\em repetitive}---the same reward $r$ on every round. If the player knows that the reference arm is repetitive, it should not be difficult for the player to achieve sublinear regret, even if he does not know what $r$ is. He can start with an {\em exploration phase} (occasional {\sf switch} requests embedded in sequences of {\sf stay} actions) that will alert him to repeated patterns of $r$ values in-between two switches. Thereafter, in an {\em exploitation phase}, whenever the player gets a reward below $r$, he will ask for a switch. The only way the decoy arm can cause the player not to reach the reference arm is by offering rewards higher than $r$, but getting rewards higher than $r$ on the decoy arm causes no regret. (The above informal argument is made formal in the proof of \cref{thm:positive}.)

The above argument can be extended (with an $O(\epsilon T)$ loss in the regret) to the case that the rewards on the reference arm are {\em $\epsilon$-repetitive}, namely, in the range of $r \pm \epsilon$ for some $r$.
Suppose now that given some integer $d < T$, the reference arm is not $\epsilon$-repetitive, but only {\em $(d,\epsilon)$-locally repetitive}, in the following sense: starting at any round that is a multiple of $d$, the sequence of rewards on the $d$ rounds that follows is $\epsilon$-repetitive. A $(d,\epsilon)$-locally repetitive sequence need not be $\epsilon$-repetitive---it can change values arbitrarily every $d$ rounds. However, if $d$ is sufficiently large (compared to $1/p$ in the hidden bandit setting), the player should be able to achieve small regret, by breaking the sequence of length $T$ to $T/d$ blocks of size $d$, and treating each block as an $\epsilon$-repetitive sequence.

But what happens if the rewards on the reference arm are not $(d,\epsilon)$-repetitive? Then we can use a notion of {\em scales}. For $0 \le \ell < \log_d T$, the scale-$\ell$ version of a sequence of length $T$ is obtained by bunching together groups of $d^{\ell}$ consecutive rounds into one super-round, and making the reward of the super-round equal to the average of the rewards of the rounds it is composed of. The player in the hidden bandit setting may choose a random scale $\ell$, in hope that in this scale the resulting sequence of super rounds is $(d,\epsilon)$-repetitive.
It turns out this approach works. This is a consequence of the {\em local repetition lemma} that we state here informally.

\begin{repthm}[Local repetition lemma, informal statement]{lem:localrepetitive}
For every choice of integer $d \ge 2$ and $0 < \epsilon, \delta < 1$, if $T$ is sufficiently large (as a function of $d$, $\epsilon$ and $\delta$), then for every string in $\sigma \in [0,1]^T$, in almost all scales (say, a fraction of $1 - \delta$) the resulting sequence is almost $(d,\epsilon)$-repetitive (almost in the sense that only a $\delta$ fraction of the blocks fail to be $\epsilon$-repetitive).
\end{repthm}

We are not aware of a previous formulation of the local repetition lemma.
However, it has connections to results that are well known in other contexts. We briefly mention several such connections, without attempting to make them formal. The regularity lemma of Szemer\'{e}di asserts that every graph has some ``regular" structure. Likewise, the local repetition lemma asserts that every string has some ``regular" (in the sense of being nearly repetitive) structure. Our proof for the local repetition lemma follows standard techniques for proving the regularity lemma, though is easier (because strings are objects that are less complicated than graphs). An alternative proof for the local repetition lemma can go through martingale theory (e.g., through the use of martingale upper-crossing inequalities). The relation of our setting to that of martingales is that the sequence of values observed when going from a super round in the highest scale all the way down to a random round in smallest scale is a martingale sequence. Yet another related topic is Parseval's identity for the coefficients of Fourier transforms. It gives an upper bound on the sum of all Fourier coefficients, implying that most of them are small. This means that a random scale a sequence of values has small Fourier coefficients, and small Fourier coefficients correspond to not having much variability at this scale.

Our lower bound of \cref{thm:negative} is based on a construction that was used by~\cite{dekel2013bandits} for proving lower bounds on the regret for bandit settings with switching costs. The construction is a full binary tree with $T$ leaves that correspond to the rounds, in which each edge of the tree has a random reward, and the reward at a leaf is the sum of rewards along the root to leaf path. The reward on the decoy arm is identical to that of the reference arm, except for a constant offset, which on the one hand should not be too large so that the player cannot tell when he is switching between arms, and on the other hand should not be too small as it determines the regret. In the context of~\cite{dekel2013bandits}, such a construction results in a regret of $\Omega(T^{2/3}/\log{T})$. In our context, a similar construction gives a much higher, almost linear lower bound. 
We remark that our modification of this randomized construction 
share similarities with a construction used by \cite{dwork2010differential} to obtain positive results in a different context, that of {\em differential privacy}. (The inability of the player to distinguish between the reference arm and the decoy arm is analogous to keeping the value of an offset ``differentially private".)

The upper bound in \cref{thm:semimarkov} is based on a simple randomized algorithm that in every round asks for a switch with probability that is exponential in the negative of the reward of that particular round. The proof that this algorithm has low regret (when the adversary is consistent) is based on showing that the expected fraction of rounds spent on the decoy arm is exponential in the (negative) offset of the decoy arm compared to the reference arm.  

The lower bound in \cref{thm:semimarkov} (against semi-Markovian algorithms) is based on the adversary choosing at random a fixed reward on the reference arm and a fixed smaller reward on the decoy arm. Natural distributions for choosing these two rewards only lead to a regret that behaves roughly like $\Omega(T/\log^{3/2}{T})$. To get the matching lower bound of $\Omega(T/\log T)$ we use a distribution similar to the distribution of queries that was used in work of~\cite{raskhodnikova1999monotonicity} on monotonicity testing with a small number of queries.

\subsection{Extensions of Our Upper Bound} 
\label{sec:extensions}

We discuss a few simple extensions of the basic model presented in \cref{sec:model}. 

\paragraph{Time-dependent policies.}
In our stateful policies model, reference policies were assumed to be time independent. We may also consider a model in which reference policies can be time dependent (the functions $f^{\pi},g^{\pi}$ have an additional input which is the round number). Our lower bound (\cref{cor:negative}) is proved with respect to time independent reference policies, and hence holds without change when reference policies can be time dependent. Our upper bound (\cref{cor:positive}) also holds without change when reference policies are time dependent---nothing in the proof of \cref{cor:positive} requires time independence.

\paragraph{Randomized policies.}
In our stateful policies model, reference policies were assumed to be deterministic. We may also consider a model in which reference policies can be randomized (the functions $f^{\pi},g^{\pi}$ have access to random coin tosses).  Our lower bound (\cref{cor:negative}) is proved with respect to deterministic reference policies, and hence holds without change when reference policies can be randomized. For the upper bound, there are two natural ways of evaluating the regret. One, less demanding, is against the expected total reward of the reference policy with highest total expected reward. The other, more demanding, is against the expectation of the realized maximum of the total rewards of the reference policies. (That is, one runs each one of the reference policies using independent randomness, and observes which policy achieves the highest reward.) Our upper bound (\cref{cor:positive}) extends to randomized reference policies, even under the more demanding interpretation---one simply fixes for each reference policy all its random coin tosses in advance, thus making it deterministic, and then \cref{cor:positive} applies with no change.

\paragraph{Stateful and reactive adversaries.}
One of the motivations of the current study was to consider also stateful adversaries, and not just stateful policies. For a stateful adversary, the reward at a given round can depend not only on the action taken by the player, but also on the entire history of the game up to that round (via some state variable that the adversary keeps and updates after every round). In general, it is hopeless to attain sublinear regret in such settings (for example, the action taken in the first round might determine the rewards in all future rounds, and then one mistake by the player already gives linear regret). However, our positive results do extend to a certain class of stateful adversaries, for which the reward received at any round is a function of the actions of the player on that and the $\ell$ previous rounds (for some fixed $\ell$).
We refer to this class as {\em reactive adversaries}, in analogy to our notion of reactive policy, though it has been studied in the literature under the names ``loss functions with memory'' \citep{merhav2002sequential} and ``bounded memory adaptive adversary'' \citep{arora2012online}.
See \cref{sec:reactive} for more details.

\section{Proofs}

\subsection{The Local Repetition Lemma} \label{sec:localrepeat}

In this section we formulate and prove the local repetition lemma, which is a key lemma for the proof of Theorem~\ref{thm:positive}. As this lemma may have other applications, we use a generic terminology that is not specific to our sequential decision models.
In the notation of the local repetition lemma, a sequence will be referred to as a string, its length will typically be denoted by $n$ (rather than $T$), and the entries of the string (which will still have values in $[0,1]$) will be referred to as characters rather than rewards. 
Hence, strings are concatenation of individual characters, where the value of a character is a real number in the range $[0,1]$. 
However, it will be convenient for us to sometimes view a string as a concatenation of substrings.
Namely, each entry of the string might be a string by itself, and the whole string is a concatenation of these substrings. We may apply this view recursively, namely, the entries of each substring might also be substrings rather than individual characters. The notation that we introduce below is flexible enough to encompass this view.

For arbitrary $n$, given a string $s \in [0,1]^n$, $x_s$ denotes its average value. Using $s(i)$ to denote the $i$th entry of $s$, and using $x_{s(i)}$ to denote the value of this entry, we thus have $x_s = \frac{1}{n}\sum_{i=1}^n x_{s(i)}$. This notation naturally extends to the case that $s$ is not a string of characters, but rather a string of $n$ substrings, in which each substring $s(i)$ is by itself a string of $m$ characters (same $m$ for every $1 \le i \le n$). In this case, $x_{s(i)}$ is the average value of string $s(i)$, and the expression $\frac{1}{n}\sum_{i=1}^n x_{s(i)}$ still correctly computes $x_s$.

As a rule, whenever we view a string as being composed of substrings, all these substrings will be of exactly the same length.

\begin{definition}[repetitive string]
\label{def:repetitive}
Let $n$ be a multiple of $d$. Consider a string $s \in [0,1]^n$, viewed as a concatenation of $d$ substrings, $s(1), \ldots, s(d)$, each in $[0,1]^{n/d}$. Given $\epsilon > 0$, we say that $s$ is {\em $(d,\epsilon)$-repetitive} if for every $i$ we have $|x_s - x_{s(i)}| \le \epsilon$.
\end{definition}

A key aspect of our approach is that we shall typically not consider the string as a whole, but rather consider only a {\em local} portion of the string, namely, a substring. Moreover, the size of the local portion depends on the level of resolution at which we wish to view the string. Consequently, we endow the string with a probability distribution over its substrings, as in \cref{def:sampling}.

\begin{definition}[$d$-sampling]
\label{def:sampling}
Let $n$ be a power of $d$, say $n = d^k$. A {\em $d$-sampling} of a string $s \in [0,1]^n$ proceeds as follows. First a value $\ell$ (for {\em level}) is chosen uniformly at random from $\set{0,\ldots,k-1}$. Then $s$ is partitioned into $d^{\ell}$ consecutive substrings, each of length $d^{k-\ell}$. Thereafter, one of these substrings is chosen uniformly at random, and declared the result of the sampling.
\end{definition}

The result of $d$-sampling is always a string whose length is divisible by $d$, and hence compatible in terms of length with the requirements of \cref{def:repetitive}.

\begin{remark*}
In \cref{def:sampling} we assume that $n$ is a power of $d$. We shall make similar simplifying assumptions throughout this section. However, our work easily extends to cases that $n$ is not a power of $d$. We explain how to do this in the context of $d$-sampling. Let $k$ be largest such that $d^k \le n$. With probability $d^{k}/n$ choose the prefix of length $d^k$ of $s$ and on it do $d$-sampling as in \cref{def:sampling}. With the remaining probability $1 - d^{k}/n$ choose the suffix of length $n-d^k$ of $s$, and recursively partition it into a prefix and suffix as above, applying \cref{def:sampling} only to the prefix. When the suffix becomes shorter than $d$, stop (this suffix can be discarded from $s$ without affecting our results).
\end{remark*}

We can now state the key definition for this section.

\begin{definition}[locally-repetitive string]
\label{def:localrepetitive}
Let $n$ be a power of $d$, and consider a string $s \in [0,1]^n$. Given $\epsilon, \delta > 0$, we say that $s$ is {\em $(d,\epsilon,\delta)$-locally-repetitive} if with probability at least $1 - \delta$, a random substring of $s$ sampled using $d$-sampling (as in \cref{def:sampling}) is $(d,\epsilon)$-repetitive (as in \cref{def:repetitive}).
\end{definition}

The main result of this section is the following.

\begin{lemma}[Local repetition lemma]
\label{lem:localrepetitive}
Let $d$ be a positive integer, and $\eps, \del > 0$. Then for every $n > d^k$ where $k = d/(4\eps^2 \del)$, every string $s \in [0,1]^n$ is $(d,\eps,\del)$-locally repetitive.
\end{lemma}

\begin{proof}
For simplicity, we shall assume that $1/\eps$ and $1/\del$ are integers.
Let $s$ be a string in $[0,1]^n$ with $n = d^k$. We say that a substring $v$ is {\em aligned} if its location in $s$ is such that it may be obtained as a result of $d$-sampling. Observe that if $v$ is aligned, then it is a concatenation of $d$ equal length strings $v(1), \ldots, v(d)$, each of which is aligned as well.
Recall that we refer to $\ell$ in \cref{def:sampling} as the {\em level}. We use the notation $v \in \ell$ to say that $v$ is aligned, and moreover, $v$ is in level $\ell$ with respect to $d$-sampling.

Define the {\em variability} of level $\ell$ to be:
$$
	\forall ~ 0 \le \ell \le k ~,
	\qquad
	V_{\ell} \eq
	\frac{1}{d^{\ell}}\sum_{v \,\in\, \ell} (x_v)^2 ~.
$$

\begin{proposition}
\label{pro:variability}
With the above definition, we have $V_{k} - V_{0} \le \frac{1}{4}$.
\end{proposition}

\begin{proof}
By definition, $V_{0} = (x_s)^2$. Observe that $V_{k}$ is maximized if all characters are 0/1, and moreover, the average of a character in level $k$ is exactly $x_s$. Hence, $V_{k} \le x_s$. The difference $V_{k} - V_{0} = x_s - (x_s)^2$ is maximized when $x_s = \frac{1}{2}$, giving a value of $\frac{1}{4}$.
\end{proof}

$V_{\ell}$ is monotonically nondecreasing with $\ell$, because for a given aligned string $v$ with substrings $v(1), \ldots, v(d)$, we have that $x_v = \frac{1}{d}\sum_{i=1}^d x_{v(i)}$, and the square of an average is never larger than the average of the squares.  For aligned strings $v$ that are not $(d,\epsilon)$-repetitive, the following proposition shows that there is a noticeable increase in variability in the next level.

\begin{proposition}
\label{pro:badvertex}
If $v$ is an aligned string that is not $(d,\epsilon)$-repetitive, then 
$$
	\frac{1}{d}\sum_{i=1}^d (x_{v(i)})^2 
	~>~ (x_v)^2 + \frac{\epsilon^2}{d} ~.
$$
\end{proposition}

\begin{proof}
If $v$ is not $(d,\epsilon)$-repetitive, then it has at least one substring $v(i)$, with $|x_v - x_{v(i)}| > \epsilon$. Hence $\sum_{i=1}^d (x_v - x_{v(i)})^2 > \epsilon^2$. By definition, $x_v = \frac{1}{d}\sum_{i=1}^d x_{v(i)}$.
Hence $\sum_{i=1}^d (x_v - x_{v(i)})^2 = \sum_{i=1}^d (x_{v(i)})^2  - d(x_v)^2$.
Putting these two facts together we get that $\sum_{i=1}^d (x_{v(i)})^2 > d(x_v)^2 + \epsilon^2$, implying the proposition.
\end{proof}

Let $\delta_{\ell}$ be the conditional probability that given that the $d$-sampling procedure sampled level $\ell$, the substring $v$ sampled is not $(d,\epsilon)$-repetitive. Hence $\delta = \frac{1}{k}\sum_{\ell=0}^{k-1} \delta_{\ell}$. Then applying \cref{pro:badvertex} level by level implies that 
$$
	V_{k} - V_{0}
	~>~ \frac{\epsilon^2}{d}\sum_{\ell=0}^{k-1} \delta_{\ell} 
	\eq \frac{k\delta \epsilon^2}{d} ~.
$$
Contrasting this with \cref{pro:variability} we obtain that $\frac{k\delta \epsilon^2}{d} < \frac{1}{4}$, implying that $\delta < \frac{d}{4\epsilon^2 k}$.
\end{proof}

In \cref{sec:doob} we provide an alternative proof for \cref{lem:localrepetitive}. Though that proof gives somewhat weaker bounds, we find it informative, as it shows the relation between \cref{lem:localrepetitive} and martingale theory.

\cref{lem:localrepetitive} is best possible in the following sense.

\begin{lemma}
\label{lem:notlocalrepetitive}
There is a universal constant $c > 0$ such that the following holds. Let $d$ be a positive integer, and $0 < \epsilon, \delta < \frac{1}{2}$. 
Then there exists a string $s \in [0,1]^n$,  where $n > d^k$ with $k = cd/(\epsilon^2 \delta)$, that is not $(d,\epsilon,\delta)$-locally repetitive.
\end{lemma}

\begin{proof}
Again, we assume for simplicity that $1/\eps$ and $1/\del$ are integers.
Fix $d, k$ let $n = d^k$. Given $\epsilon > 0$, pick $\eta > \epsilon$ to be as small as possible, conditioned on $1/2\eta$ being an integer. Construct a string $s \in [0,1]^n$ in a top-down manner, by associating the $x_v$ variables with the possible choices of substrings $v$ in the $d$-sampling scheme. Start by setting $x_s = \frac{1}{2}$. Thereafter, for every substring $v$ for which $x_v$ is already determined do the following. If $v$ is a single character, nothing needs to be done. Else, $v$ represents a string whose length is a multiple of $d$. Let $v(0), \ldots , v(d-1)$ denote the $d$ substrings whose concatenation gives $v$. If either $x_v = 0$ or $x_v = 1$, then for every $i$, let $x_{v(i)} = x_v$. In this case $v$ is $(d, \epsilon)$-repetitive. However, in every other case, let $x_{v(0)} = x_v + \eta$, $x_{v(1)} = x_v - \eta$, and $x_{v(m)} = x_v$ for all $2 \le m \le d-1$. In this case $v$ is not $(d,\epsilon)$-repetitive.

The construction above maintains that $0 \le x_v \le 1$ for every $v$, and moreover, $x_v = \frac{1}{d}\sum_{i=1}^d x_{v(i)}$. Hence the $x_v$ variables indeed represent the true averages over the corresponding substrings of $s$.

For an individual character at location $i$ in the string $s$, its value is integer (0 or~1) if and only if when writing $i$ in base $d$ (namely, $i = a_{k-1}d^{k-1} + \ldots a_1 d + a_0$, with $0 \le a_j \le d-1$), there is a value $0 \le j \le k-1$ such that among the coefficients $a_j, \ldots , a_{k-1}$, there is a difference of at least $1/\eta$ between the number of those coefficients that are~0 and the number of those coefficients that are~1.

Let us now set $k = cd/\eta^2$, for some small universal constant $c > 0$ (independent of $d$ and $\eta$). It is not difficult to argue that in this case, only a small fraction of the characters of $s$ are integer (0 or~1). (For a random $i$, only $O(k/d)$ of its digits in base $d$ are 0/1, and for most sequences of $\pm 1$ of length $m$, there is no prefix whose sum exceeds $O(\sqrt{m})$ in absolute value.) Hence for this value of $k$, almost no $v$ is $(d,\epsilon)$-repetitive.

Finally, set $k = k_0/2\del$, with $k_0 = cd/\eta^2$. With probability $2\delta$ the $d$-sampling procedure will choose a level among the top $k_0$ levels, and then with probability at least $\frac{1}{2}$ the sampled string $v$ will not be $(d,\epsilon)$-repetitive.
\end{proof}

\subsection{Upper Bound for Hidden Bandits}
\label{sec:bandit-upper}

In this section we present an algorithm for the hidden bandit problem whose worst-case expected regret is sublinear.
Our algorithm exploits the fact that the reward sequence of the reference arm, whose values are set in an oblivious manner by the adversary, is $(d,\eps,\del)$-locally repetitive (see \cref{def:localrepetitive}) for appropriately chosen values of $d,\eps,\del$, as implied by the local repetition lemma. 
Hence, it would be instrumental to first consider the simpler case where the reference sequence is in fact $(d,\eps)$-repetitive (see \cref{def:repetitive}).

When the reference sequence is $(d,\eps)$-repetitive, we propose an algorithm, described in \cref{alg:onearm}, which is based on a simple first-explore-then-exploit strategy.
The algorithms begins with an exploration phase (Phase I), where it tries to hit the reference arm at least once and obtain an estimate of its reward, which is almost constant at the appropriate scale.
Then, in the exploitation phase (Phase II), the algorithm repeatedly asks for a \textsf{switch} whenever the observed rewards drops below the top estimated rewards obtained in Phase I.
Eventually, since the reference arm is $(d,\eps)$-repetitive, the algorithm should stabilize on that arm no matter what the rewards on the decoy arm are. 

\begin{algorithm}
\begin{framed}
\paragraph{Algorithm for $(d,\eps)$-repetitive reference sequence } (parameters: $d, \eps, p, T$) 
\begin{itemize}[leftmargin=\algindent]
\item
let $m = (1/p) \log(1/\eps)$
\item
\textbf{Phase I:} for $i=1,\ldots,m$: \textsf{stay} on chosen arm for $T/d$ rounds and let $\bar{r}_{i}$ be the average of the observed rewards, then \textsf{switch} once
\item
sort the averages $\bar{r}_{1},\ldots,\bar{r}_{m}$ in descending order to obtain $\bar{r}_{1} \ge\ldots \ge \bar{r}_{m}$
\item
\textbf{Phase II:} initialize $i = 1$, $s=0$ and repeat (until $T$ rounds have elapsed):
\begin{itemize}[leftmargin=\algindent]
\item
\textsf{stay} for $T/d$ rounds and let $\bar{r}$ be the average of the observed rewards
\item
if $\bar{r} < \bar{r}_{i} - 2\eps$, \textsf{switch} once and update $s \gets s+1$
\item
if $s \ge m$, update $i \gets i+1$ and reset $s = 0$
\end{itemize}
\end{itemize}
\vspace{-2ex}
\end{framed}
\caption{An algorithm for $(d,\eps)$-repetitive reference sequences.} \label{alg:onearm}
\end{algorithm}

The following lemma shows that for small values of $\eps$, if $d$ is large enough as a function of $\eps$ then the expected regret of \cref{alg:onearm} is not large.

\begin{lemma} \label{lem:onearm}
Assume that the reward sequence of the reference arm is $(d,\eps)$-repetitive, with $d \ge (1/p^{2} \eps) \log^{2}(1/\eps)$.
Then the expected regret of \cref{alg:onearm} is at most $8\eps T$.
\end{lemma}

\begin{proof}
Let $v$ denote the average reward (over all rounds $t=1,2,\ldots,T$) of the reference arm.
Notice that the probability of not visiting this arm in the first phase of the algorithm is   no more than
$
	(1-p)^{m} \le e^{-p m} = \eps .
$
Hence, with probability at least $1-\eps$, the first phase of the algorithm samples the reference arm at least once, so that there exists some $j \in [m]$ for which $\bar{r}_{j} \in [v-\eps, v+\eps]$ as the reward sequence of the reference arm is assumed to be $(d,\eps)$-repetitive.
The total regret incurred in this phase is bounded by $m \cdot T/d$.

Next, assume that indeed $\bar{r}_{j} \in [v-\eps, v+\eps]$ for some $j \in [m]$ and consider the second phase of the algorithm.
Notice that once $i = j$ and the reference arm is selected, the algorithm stops switching and stays on that arm until the game ends.
This is true because the reference arm is $(d,\eps)$-repetitive, so each average reward $\bar{r}$ encountered when this arm is selected exceeds $v - \eps \ge \bar{r}_{j} - 2\eps$.

Let us bound the regret incurred in the second phase of the algorithm until this event occurs (if at all).
To this end, consider iterations with $i \le j$.
On any such iteration in which $\bar{r} \ge \bar{r}_{i} - 2\eps$, we have $\bar{r} \ge \bar{r}_{j} - 2\eps \ge v-3\eps$ so that the incurred regret is at most $4\eps \cdot T/d$. 
The number of iterations that fail to satisfy this condition is equal to the number of \textsf{switch} actions issued by the algorithm.
The number of \textsf{switch} actions in iterations with $i < j$ is no more than $(j-1) \cdot m$,
and once $i = j$, the algorithm hits the reference arm after at most $m$ \textsf{switch} actions with probability $1-\eps$ (and subsequently stops switching).
Thus, with probability at least $1-\eps$ the total number of \textsf{switch} actions is bounded by $(j-1) \cdot m + m \le m^{2}$.
In this case, the total regret incurred in the second phase is at most $4\eps T + m^{2} \cdot T/d$. 
Overall, the total regret in both phases is then bounded by
\begin{align*}
	4\eps T + \frac{T}{d} \cdot m^{2} + \frac{T}{d} \cdot m
	\leq 4\eps T + \frac{2T}{d} \cdot m^{2}
	\leq 4\eps T + \frac{2T}{d} \cdot \frac{1}{p^{2}} \log^{2} \frac{1}{\eps}
	\leq 6\eps T ~,
\end{align*}
where we have used our assumption that $d \ge (1/p^{2} \eps) \log^{2}(1/\eps)$.

On the other hand, if one of the phases fail then the total regret might be as large as $T$, but this happens with probability at most $2\eps$.
Hence, the expected regret of the algorithm is at most $8\eps T$.
\end{proof}

Our general algorithm, that works for any reference sequence, is described in \cref{alg:onearm2}.
The algorithm invokes \cref{alg:onearm} above as a subroutine on a randomly-chosen block size, exploiting the locally-repetitive structure guaranteed by the local repetition lemma.

\begin{algorithm}
\begin{framed}
\paragraph{Algorithm for hidden bandits}  (parameters: $p, T$)
\begin{itemize}[leftmargin=\algindent]
\item
set
$$
	\eps \eq 
	\frac{1}{\sqrt{p}} \cdot \frac{\log\log T}{\log^{1/4} {T} }
	~,\qquad
	d \eq 
	\frac{1}{p^{2}} \log^{2}\frac{1}{\eps}
$$
\item
choose block size $b=d^{i}$, where $i$ is chosen uniformly at random from $\set{1,\ldots,\floor{\log_{d}{T}}}$
\item
for $i=1,\ldots,T/b$: invoke \cref{alg:onearm} on a block of size $b$ with parameters $d,\eps,p,b$
\end{itemize}
\vspace{-2ex}
\end{framed}
\caption{An algorithm for the hidden bandit problem that guarantees sublinear expected regret.} \label{alg:onearm2}
\end{algorithm}

We are now ready to give the main result of this section, which gives an upper bound over the expected regret of \cref{alg:onearm2}.

\begin{repthm}[restated]{thm:positive}
The expected regret of \cref{alg:onearm2} is 
$$
	O \lr{ \frac{1}{\sqrt{p}} \cdot \frac{T \log\log{T}}{\log^{1/4}{T}} } ~.
$$
\end{repthm}

The proof uses the local repetition lemma, restated here for convenience.

\begin{repthm}[restated]{lem:localrepetitive}
Let $d$ be a positive integer and $\eps, \del > 0$. 
Then for every $n > d^k$ where $k = d/(4\epsilon^2 \delta)$, every string $s \in [0,1]^n$ is $(d,\epsilon,\delta)$-locally repetitive.
\end{repthm}

\begin{proof}[Proof of \cref{thm:positive}]
Set $d = (1/p^{2} \eps) \log^{2}(1/\eps)$, $\del = \eps$, $k = d / 4\eps^{3}$ in \cref{lem:localrepetitive}, which then states that any sequence of length at least
$
	T_{\eps} 
	= d^{k}
$ 
is $(d,\eps,\eps)$-locally repetitive. 
Notice that for $T \ge \Omega(1)$ and our choice of $\eps$ we have
\begin{align*}
	\log T_{\eps}
	&\eq \frac{1}{4 p^{2} \eps^{4}} \log^{2} \lr{ \frac{1}{\eps} } \cdot \log\lr{ \frac{1}{p^{2} \eps} \log^{2}\frac{1}{\eps} } \\
	&\leq \frac{1}{p^{2} \eps^{4}} \log^{4} \lr{ \frac{1}{p^{2}\eps^{4}} }
	\eq \frac{\log T}{\log^{4}(\log T)} \cdot \log^{4} \lr{ \frac{\log T}{\log^{4}(\log T)} } \\
	&\leq \log T ~,
\end{align*}
thus $T_{\eps} \le T$ which means that the reward sequence of the reference arm is $(d,\eps,\eps)$-locally repetitive.
Since $b$ was chosen uniformly at random, this means that each $b$-aligned block of size $b$ in this reward sequence is $(d,\eps)$-repetitive with probability at least $1-\eps$.

Now, consider a certain iteration of the algorithm. 
With probability $1-\eps$, the corresponding block in the reference reward sequence is $(d,\eps)$-repetitive with $d = (1/p^{2} \eps) \log^{2}(1/\eps)$.
Hence, according to \cref{lem:onearm}, following the strategy of \cref{alg:onearm} in this block yields an expected regret of $O(\eps b)$. 
Overall, the expected regret in all $T/b$ blocks is then $O(\eps T)$. 
Using the definition of $\eps$ concludes the proof.
\end{proof}

\subsection{Upper Bound for Stateful Policies}
\label{sec:policy-upper}

We now show how our algorithm for the hidden bandit setting can be applied, via a simple reduction, in the stateful policy model.
The resulting algorithm is presented in \cref{alg:policies}, that provides implementations of the \textsf{stay} and \textsf{switch} actions of the hidden bandit model.
The basic idea is to think of the best performing policy (in hindsight) within the set $\Pi$ as the reference arm, and let the decoy arm capture all other policies, as well as other action paths that do not correspond to any policy.

\begin{algorithm}
\begin{framed}
\paragraph{Algorithm for competing with stateful policies}  (parameters: $\Pi, T$)
\begin{itemize}[leftmargin=\algindent]
\item
choose a policy $\pi_{1} \in \Pi$ and a state $s_{1} \in [S]$ uniformly at random
\item
invoke \cref{alg:onearm2} with parameters $p = 1/kS$ and $T$,
and the following implementation of \textsf{stay} and \textsf{switch}:
\begin{itemize}[leftmargin=\algindent]
\item
\textsf{stay} on round $t$: 
play the action $f^{\pi_{t}}(s_{t})$, observe reward $r$, and update $\pi_{t+1} \gets \pi_{t}$ and $s_{t+1} \gets g^{\pi_{t}}(s_{t}, r)$
\item
\textsf{switch} on round $t$: 
play the action $f^{\pi_{t}}(s_{t})$, then choose a policy $\pi_{t+1} \in \Pi$ and a state $s_{t+1} \in [S]$ uniformly at random
\end{itemize}
\end{itemize}
\vspace{-2ex}
\end{framed}
\caption{An algorithm for competing with stateful policies.} \label{alg:policies}
\end{algorithm}

We now prove our main upper bound result, which provides a regret guarantee for \cref{alg:policies}.

\begin{repthm}[restated]{cor:positive}
For any reference set $\Pi$ of $k$ stateful policies over $S$ states, the expected regret of \cref{alg:policies} with respect to $\Pi$ is
$$
	O\lr{ \sqrt{kS} \cdot \frac{T \log\log{T} }{ \log^{1/4}{T} } } ~.
$$
\end{repthm}

\begin{proof}
Let $\pi^{*} \in \Pi$ be the best policy in the set $\Pi$, namely, the one having the highest total reward in hindsight.
For all $t=1,2,\ldots,T$, we let $s^{*}_{t} \in [S]$ denote the state visited by $\pi^{*}$ on round $t$ had it been followed from the beginning of the game.
Consider a hidden bandit problem where the reward sequence of the reference arm is the sequence obtained by following the policy $\pi^{*}$ throughout the game, and the arm being pulled on round $t$ is given by the random variable
$$
	\forall ~ t
	\qquad
	X_{t} \eq 
	\ind{ \pi_{t} \ne \pi^{*} \,\vee\, s_{t} \ne s^{*}_{t} } ~.
$$
The decoy arm models any situation where the algorithm deviates from the policy $\pi^{*}$, and each reward obtained on that arm is possibly a function of the entire history of the game, including even the random bits used by the player.
Since the model allows for the decoy arm to be completely arbitrary, we do not precisely specify the rewards associated with that arm. 
The claimed regret bound would then follow from \cref{thm:positive} once we verify that the implementations of the \textsf{stay} and \textsf{switch} actions are correct, namely:
\begin{enumerate}[(i)]
\item
if $X_{t} = 0$ (i.e., the algorithm is on the reference arm on round $t$), then choosing \textsf{stay} ensures that $X_{t+1} = 0$;
\item
if $X_{t} = 1$ and the algorithm chooses \textsf{switch} then $X_{t+1} = 0$ with probability at least $p = 1/kS$.
\end{enumerate}
Again, since the decoy arm may be completely adversarial, it is not crucial to verify the transitions directed towards it (in particular, the decoy arm might imitate the reference arm in response to a certain action of the algorithm).

To see (i), note that $X_{t} = 0$ implies $\pi_{t} = \pi^{*}$ and $s_{t} = s^{*}_{t}$.
In particular, the algorithm picks on round $t$ the same action played by $\pi^{*}$ on that round and observes the same reward.
Hence, if the algorithm chooses \textsf{stay} then the update $s_{t+1} \gets g^{\pi_{t}}(s_{t}, r)$ ensures that $s_{t+1} = s^{*}_{t+1}$, retaining the algorithm in the correct state on round $t+1$.
Next, if $X_{t} = 1$ which means that the algorithm is not on the reference arm on round $t$, then by choosing \textsf{switch} the random choice of $(\pi_{t+1}, s_{t+1})$ hits the configuration $(\pi^{*}, s^{*}_{t+1})$ with probability $p = 1/kS$. 
That is, with probability at least $1/nS$ the algorithm would be on the reference arm on round $t+1$, which proves (ii).
\end{proof}

\begin{remark*}
Following the same idea explained in the proof above, it is actually possible to obtain a slightly improved dependence on the number of policies $n$ and save a $n^{1/4}$ factor in the resulting bound, albeit with a more involved algorithm.
\end{remark*}

\subsection{Lower Bound for Hidden Bandits}
\label{sec:bandit-lower}

In this section we prove our lower bound for the hidden bandit problem with $p=1/2$ given in \cref{thm:negative}, which we restate here more formally. 

\begin{repthm}[restated]{thm:negative}
For any randomized player strategy in the hidden bandit model with $p=1/2$, there exists an oblivious sequence of reward functions $r_{1},\ldots,r_{T}$ that forces the player to incur an expected regret of $\Omega(T/\log^{3/2}{T})$ with respect to the reference arm.
\end{repthm}

In order to prove \cref{thm:negative} we make use of Yao's principle \citep{Yao77}, which in our context states that the expected regret of a randomized algorithm on the worst case reward sequence is no better than the expected regret of the optimal deterministic algorithm on any stochastic reward sequence. 
Hence, \cref{thm:negative} would follow once we establish the existence of a single sequence of stochastic reward functions, $\Gamma_{1:T}$, which is difficult for any deterministic algorithm of the player (in terms of expected regret).

Our construction of the required stochastic sequence $\Gamma_{1:T}$ is based on a variant of the Multi-scale Random Walk stochastic process of \cite{dekel2013bandits}.

\begin{definition}[Multi-scale Random Walk, \citealt{dekel2013bandits}]
Given a sequence $\xi_{1}, \ldots, \xi_{T}$ of i.i.d.~random variables, the \emph{Multi-scale Random Walk} (MRW) process $W_{0:T}$ is defined recursively by
\begin{align} \label{eq:rw}
	W_0 &\eq 0 ~, \non\\
	\forall~t \in [T] \qquad
	W_t &\eq W_{\rho(t)} + \xi_t ~, 
\end{align}
where
$$
	\rho(t) = t - 2^{\delta(t)} ~,\quad
	\delta(t) = \max\set{i \ge 0 : 2^i \text{ divides } t} ~.
$$
\end{definition}

Our construction, described in \cref{fig:lowerbound}, is similar to the one used by \cite{dekel2013bandits}, with one crucial difference: instead of using a Gaussian distribution for the step variables $\xi_{1:T}$, we employ a two-sided geometric distribution supported on integer multiples of $\eps$ (this is a discrete analog of the continuous Laplace distribution).
We then use the resulting MRW process $W_{1:T}$ to form a sequence of intermediate reward functions $\wt\Gamma_{1:T}$, where the reward of arm $x=0$ is consistently better than that of arm $x=1$ by a gap of $\eps$.
The actual reward functions $\Gamma_{1:T}$ are obtained from $\wt\Gamma_{t}$ by clipping the reward values to the $[0,1]$ interval.

\begin{figure}[t]
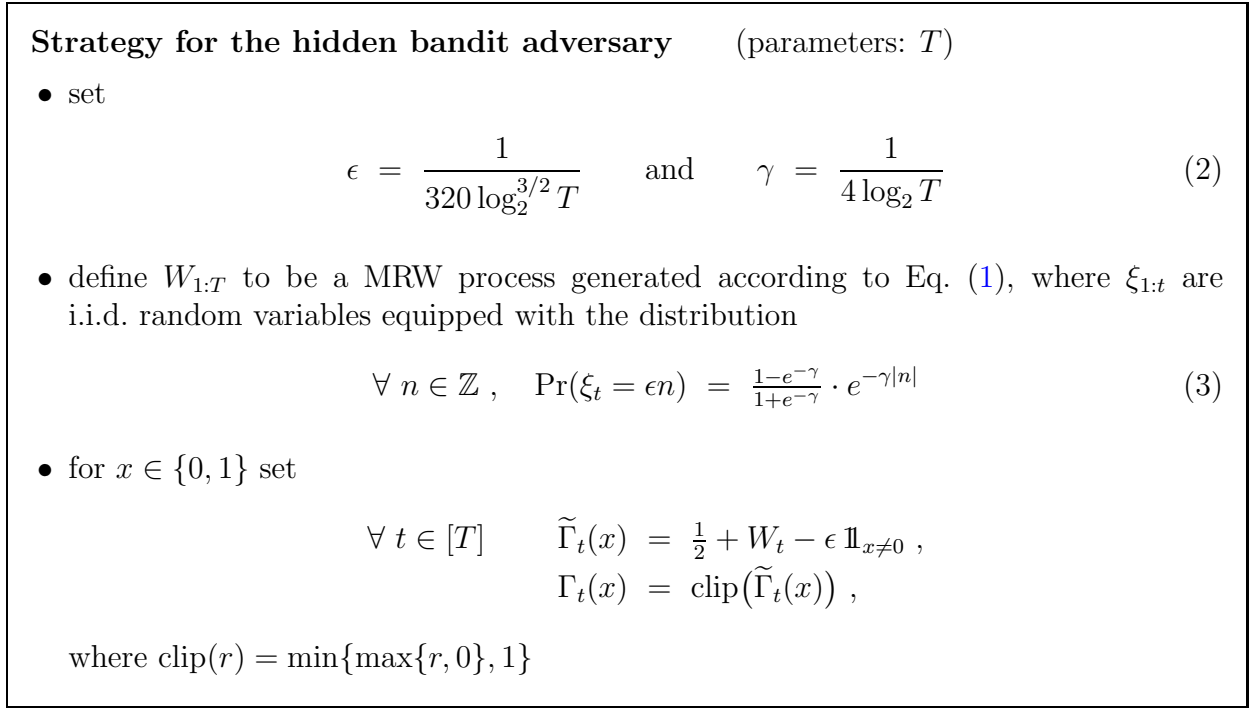

\begin{framed}
\paragraph{Strategy for the hidden bandit adversary} \quad (parameters: $T$)

\begin{itemize}[leftmargin=\algindent]
\item
set
\begin{align} \label{eq:eps-gam}
	\eps \eq \frac{1}{ 320 \log_{2}^{3/2}{T} }
	\qquad\text{and}\qquad
	\gam \eq \frac{1}{4\log_{2}{T}}
\end{align}
\item
define $W_{1:T}$ to be a MRW process generated according to \cref{eq:rw}, where $\xi_{1:t}$ are i.i.d.~random variables equipped with the distribution
\begin{align} \label{eq:xidist}
	\forall ~ n \in \mathbb{Z}~, 
	\quad
	\Pr(\xi_{t} = \eps n) 
	\eq \tfrac{1-e^{-\gam}}{1+e^{-\gam}} \cdot e^{-\gam \abs{n}}
\end{align}
\item
for $x \in \set{0,1}$ set
\begin{align*}
	\forall ~ t \in [T] ~
	\qquad
	\wt\Gamma_{t}(x) &\eq \tfrac{1}{2} + W_{t} - \eps \, \ind{ x \ne 0 } ~, \\
	\Gamma_{t}(x) &\eq \clip\lrbig{ \wt\Gamma_{t}(x) } ~,
\end{align*}
where $\clip(r) = \min\set{\max\set{r, 0},1}$
\end{itemize}
\vspace{-2ex}
\end{framed}
\caption{An oblivious strategy for the adversary that forces a regret of $\Omega(T/\log^{3/2} T)$ for any algorithm for the hidden bandit problem with $p=1/2$.} \label{fig:lowerbound}
\end{figure}

For this construction, we prove the following lower bound on the performance of any deterministic algorithm that immediately implies \cref{thm:negative}.

\begin{theorem} \label{thm:bandit-lower}
The expected regret of any deterministic player algorithm on the stochastic sequence of reward functions $\Gamma_{1:T}$ defined in \cref{fig:lowerbound} is at least $10^{-4} \cdot T/\log_{2}^{3/2}{T}.$
\end{theorem}

Before we begin with the analysis, we recall two key combinatorial properties of the MRW process that are essential to our analysis.
See \cite{dekel2013bandits} for more details and the formal proofs.

\begin{definition}[\emph{depth}]
Given a parent function $\rho$, the set of ancestors of
$t$ is denoted by~\,$\anc(t)$ and defined as the set of positive indices
that are encountered when $\rho$ is applied recursively to
$t$. Formally, $\anc(t)$ is defined recursively as
\begin{align}
	\anc(0) &\eq \{\} \non \\
	\forall~t \in [T]~~~~ 
	\anc(t) &\eq  \anc\big(\rho(t)\big)~\cup~\{\rho(t)\}~~. \label{eqn:ancestor}
\end{align}
The \emph{depth} of $\rho$ is then defined as $\depth(\rho) = \max_{t \in [T]} |\anc(t)|$.
\end{definition}

\begin{definition}[\emph{cut, width}]
Given a parent function $\rho$, define  
$$
	\cut(t) \eq \set{s \in [T] ~:~ \rho(s) < t \le s}~,
$$
the set of rounds that are separated from their parent by $t$.
The \emph{width} of $\rho$ is then defined  as $\width(\rho) = \max_{t \in [T]} |\cut(t)|$.
\end{definition}

\begin{lemma} \label{lem:mrw}
The depth and width of the MRW are both upper-bounded by $\floor{\log_2{T}}+1$.
\end{lemma}

We begin the analysis with some notation.
Fix some deterministic player strategy, that generates a sequence of random variables $X_{1:T}$ when faced with the random reward functions $\Gamma_{1:T}$, where $X_{t} \in \set{0,1}$ is the arm pulled on round $t$ of the game.
We let $Y_{t} = \wt\Gamma(X_{t})$ denote the unclipped reward encountered on round $t$ by the following the strategy (which is not directly observable to the player).

The strategy induces a partition of the rounds into \emph{epochs}, where epoch $m$ spans over rounds between the player's $m-1$ and $m$'th \textsf{switch} actions.
Let $\chi_{m} \in \set{0,1}$ denote the arm pulled throughout epoch $m$, and let $T_{m}$ denote the length of that epoch. 
We set $T_{m} = 0$ if the $m$'th epoch does not take place (that is, if the player makes less than $m-1$ \textsf{switch} actions throughout the game).
Without loss of generality, we may assume that there are exactly $T$ epochs corresponding to $m=1,2,\ldots,T$, some of which are of zero length.

Without loss of generality, we may assume that the assignment of arms to epochs is determined before the game begins.
Namely, a sequence of random variables $\chi_{1},\chi_{2},\ldots,\chi_{T}$ is drawn ahead of time from the distribution described by the Markov chain, where $\chi_{m} \in \set{0,1}$ is the index of the arm assigned to the player on the $m$'th epoch (some of these variables may eventually not be used).
Notice that, as we assume the player to be deterministic, the set of random variables $\xi_{1:T}$ and $\chi_{1:T}$ completely determine the outcome of the game.

A key step in our analysis is to show that, even if we allow the player to observe the entire sequence $Y_{1:T}$ directly, he is unable to detect (with sufficient confidence) an epoch during which the reference arm was pulled.
To this end, for each epoch $m \in [T]$ we define two conditional probability measures 
\begin{align*}
	P_{m}(\cdot) &\eq \Pr(~\cdot \mid \chi_{m} = 0) ~ \\
	Q_{m}(\cdot) &\eq \Pr(~\cdot \mid \chi_{m} = 1) ~,
\end{align*}
over the sigma algebra $\mathcal{F} = \sig(Y_{1:T})$ generated by the variables $Y_{1:T}$.
Our first lemma shows that for any event observable to the player (i.e., one that relies on the random variables the player receives as feedback) which is likely to occur assuming $\chi_{m} = 0$, is also likely to occur given $\chi_{m} = 1$.

\begin{lemma} \label{lem:info}
For all epochs $m$ and for any event $A \in \mathcal{F}$ it holds that
$
	Q_{m}(A) \ge \frac{1}{4 e} \cdot P_{m}(A) .
$
\end{lemma}

\begin{proof}
Fix some epoch $m \in [T]$. 
In order to prove the lemma, we bound the log-likelihood ratio of an arbitrary feasible realization $y_{1:T}$ of the variables $Y_{1:T}$ between the measures $P_{m}$ and $Q_{m}$.
To do that, we condition on the variables $\chi_{1},\ldots,\chi_{m-1}, \chi_{m+1}, \ldots,\chi_{T}$ corresponding to the arms pulled in other epochs.
Consider two realizations $x_{1:T}, x'_{1:T} \in \set{0,1}^{T}$ of the sequence $\chi_{1:T}$ that differ only by the value assigned to the $m$'th variable, with $x_{m} = 0$ and $x'_{m} = 1$, 
and define the measures
\begin{align*}
	P_{m}'(\cdot) &\eq P_{m}(~\cdot \mid \chi_{1:T} = x_{1:T}) ~, \\
	Q_{m}'(\cdot) &\eq Q_{m}(~\cdot \mid \chi_{1:T} = x'_{1:T}) ~.
\end{align*}
Then, we can write
\begin{align} \label{eq:loglike}
	\log\frac{P_{m}(Y_{1:T} = y_{1:T},\, \chi_{1:T}=x_{1:T})}
		{Q_{m}(Y_{1:T} = y_{1:T},\, \chi_{1:T}=x'_{1:T})}
	\eq \log\frac{ P_{m}(\chi_{1:T} = x_{1:T}) }{ Q_{m}(\chi_{1:T} = x'_{1:T}) }
		+ \log\frac{P'_{m}(Y_{1:T} = y_{1:T})}{Q'_{m}(Y_{1:T} = y_{1:T})} ~.
\end{align}
In order to bound the first term on the right-hand side, note that the Markov-chain dynamics of the \textsf{switch} action (recall \cref{fig:switch}) imply
\begin{align*}
	\frac{ P_{m}(\chi_{m+1} = x_{m+1}) }{ Q_{m}(\chi_{m+1} = x_{m+1}) }
	\eq \frac{ \Pr(\chi_{m+1} = x_{m+1} \mid \chi_{m} = 0) }
		{ \Pr(\chi_{m+1} = x_{m+1} \mid \chi_{m} = 1) }
	\leq \frac{1}{1/2}
	\eq 2 ~.
\end{align*}
Similarly, using Bayes' law,
\begin{align*}
	\frac{ P_{m}(\chi_{m-1} = x_{m-1}) }{ Q_{m}(\chi_{m-1} = x_{m-1}) }
	&\eq \frac{ \Pr(\chi_{m-1} = x_{m-1} \mid \chi_{m} = 0) }
		{ \Pr(\chi_{m-1} = x_{m-1} \mid \chi_{m} = 1) } \\
	&\eq \frac{ \Pr(\chi_{m} = 0 \mid \chi_{m-1} = x_{m-1}) }
		{ \Pr(\chi_{m} = 1 \mid \chi_{m-1} = x_{m-1}) } 
		\cdot \frac{ \Pr(\chi_{m} = 1) }{ \Pr(\chi_{m} = 0) } \\
	&\leq \frac{1/2}{1/2} \cdot \frac{2/3}{1/3}
	\eq 2 ~,
\end{align*}
where we have used the fact that the Markov chain is in its stationary distribution $(\frac{1}{3},\frac{2}{3})$.
Hence, we conclude
\begin{align} \label{eq:loglike1}
	\log\frac{ P_{m}(\chi_{1:T} = x_{1:T}) }{ Q_{m}(\chi_{1:T} = x'_{1:T}) }
	&\eq \log\frac{ P_{m}(\chi_{m-1} = x_{m-1}) }{ Q_{m}(\chi_{m-1} = x_{m-1}) }
		+ \log\frac{ P_{m}(\chi_{m+1} = x_{m+1}) }{ Q_{m}(\chi_{m+1} = x_{m+1}) } 
	\leq 2 \log{2} ~.
\end{align}

For bounding the second term on the right-hand side of \cref{eq:loglike}, we decompose it into a sum using the fact that $Y_{t}$ is conditionally independent of all $Y_{s}$ with $s \ne \rho(t)$ given $Y_{\rho(t)}$ under both $P_{m}'$ and $Q_{m}'$ (recall \cref{eq:rw}), as follows:
\begin{align*}
	\log\frac{P'_{m}(Y_{1:T} = y_{1:T})}{Q'_{m}(Y_{1:T} = y_{1:T})}
	\eq \sum_{t=1}^{T} \log\frac
		{P'_{m}(Y_{t}=y_t \mid Y_{\rho(t)}=y_{\rho(t)})}
		{Q'_{m}(Y_{t}=y_t \mid Y_{\rho(t)}=y_{\rho(t)})} ~.
\end{align*}
Here, for convenience, we define a fictitious deterministic reward $Y_{0} = y_{0} = \frac{1}{2}$.
Each term in the above sum corresponds to an edge in the dependency graph of the MRW process, formed by the function $\rho$.
Consider a particular term of the sum, that represents the edge $(\rho(t),t)$, and let $i$ and $j$ denote the epochs containing the end points $\rho(t)$ and $t$, respectively.
Notice that the conditional distribution of $Y_{t}$ given $Y_{\rho(t)}$ is determined only by the values of $\chi_{i}$ and $\chi_{j}$.
In particular, if $\chi_{i} = \chi_{j}$ then $Y_{t} = Y_{\rho(t)} + \xi_{t}$.
However, if $\chi_{i} = 0,~ \chi_{j} = 1$ then $Y_{t} = Y_{\rho(t)}+\xi_{t}-\eps$, and if $\chi_{i} = 1,~ \chi_{j} = 0$ then $Y_{t} = Y_{\rho(t)}+\xi_{t}+\eps$.
Hence, the log-likelihood term is zero unless $Y_{t} \mid Y_{\rho(t)}$ has different distributions under the measures $P_{m}'$ and $Q_{m}'$, which can happen only when either $i=m$ or $j=m$ (but not both) since the realizations $x_{1:T}, x'_{1:T}$ differ only in the value assigned to $\chi_{m}$.
In the latter case, the log-likelihood term is equal to the log-likelihood ratio between the distributions of $\xi_{t}$ and $\xi_{t} \pm \eps$ at some point in their (common) support, which can be at most $\gam$ as can be seen from \cref{eq:xidist}.

Now, notice that any edge $(\rho(t),t)$ for which either $i=m$ or $j=m$ is in $\cut(S_{m})$ or in $\cut(S_{m+1})$, where $S_{m}$ and $S_{m+1}$ denote the rounds on which epochs $m$ and $m+1$ begin.
Overall, we get
\begin{align*} 
	\log\frac{P'_{m}(Y_{1:T} = y_{1:T})}{Q'_{m}(Y_{1:T} = y_{1:T})}
	\leq \gam \cdot \abs{\cut(S_{m-1})} + \gam \cdot \abs{\cut(S_{m})}
	\leq 2\gam \, \width(\rho)
	\leq 1 ~,
\end{align*}
where we have used the fact that the width of the MRW process is upper bounded by $2\log_{2} T$ and our choice of $\gamma$ in \cref{eq:eps-gam}.
Plugging this and \cref{eq:loglike1} into \cref{eq:loglike} and exponentiating results with
$$
	P_{m}(Y_{1:T} = y_{1:T},\, \chi_{1:T}=x_{1:T})
	\geq \frac{1}{4 e} \cdot Q_{m}(Y_{1:T} = y_{1:T},\, \chi_{1:T}=x'_{1:T}) ~.
$$
Since this holds for any assignment of the variables $\chi_{i}$ ($i \ne m$), by marginalizing over these variables we obtain that 
$
	P_{m}(Y_{1:T} = y_{1:T})
	\ge \frac{1}{4 e} \cdot Q_{m}(Y_{1:T} = y_{1:T})
$
for any sequence $y_{1:T}$.
Finally, integrating this inequality over the event $A \in \mathcal{F}$ gives the lemma.
\end{proof}

We now turn back to analyzing the player's regret.
In order to lower-bound the expected regret of the player's action sequence $X_{1:T}$, it will be convenient for us to first analyze the regret of the \emph{same sequence} as measured by the unclipped reward functions $\wt\Gamma_{t}$, namely
\begin{align*}
	\wt R \eq \sum_{t=1}^{T} \wt\Gamma_{t}(0) - \sum_{t=1}^{T} \wt\Gamma_{t}(X_{t}) ~,
\end{align*}
and later deal with the effect of the clipping.
This quantity can be alternatively expressed as a simple function of the variables $T_{m}$ and $\chi_{m}$,
\begin{align*}
	\wt R \eq \sum_{m=1}^{T} \wt R_{m} ~,
	\qquad\text{where}\qquad
	\forall ~ m ~ \quad
	\wt R_{m} \eq \eps T_{m} \cdot \ind{ \chi_{m} \ne 0 } ~.
\end{align*}
Here, $\wt R_{m}$ is the regret incurred during epoch $m$ (in terms of the functions $\wt\Gamma_{1:T}$) and $\wt R$ is simply the total regret incurred in all epochs.
The following lemma relates the quantity $\E[\wt R]$ to the actual expected regret of the player.

\begin{lemma} \label{lem:boundness}
The player's expected regret can be lower-bounded as
$
	\E[R] \ge \E[\wt R] - \eps T/25 .
$
\end{lemma}

\begin{proof}
We first prove that for each $t \in [T]$, with probability at least $1-1/25$ both of the rewards $\wt\Gamma_{t}(0), \wt\Gamma_{t}(1)$ lie the interval $[0,1]$.
Then, the lemma would follow since this proves that in expectation there are at most $T/25$ rounds on which the reward functions $\Gamma_{t}(x)$ and $\wt\Gamma_{t}(x)$ do not coincide. 
On each of those rounds, the player's regret might decrease by at most $\eps$ since the gap between the arms is only narrowed by the clipping of the rewards.

To prove the above claim, fix some $t \in [T]$.
Since the depth of the process $W_{1:T}$ is bounded by $2\log_{2} T$ (see \cref{lem:mrw}), the random variable in $W_{t}$ is a sum of at most $2\log_{2}{T}$ variables of the sequence $\xi_{1:T}$.
For bounding the magnitude of each of the variables $\xi_{s}$, which are distributed according to \cref{eq:xidist}, let us first bound their variance. 
Noticing that $\abs{\xi_{s}}/\eps$ is a geometric random variable with parameter $p=1-e^{-\gamma}$ and using a standard bound of $2/p^{2}$ over the second moment of the geometric distribution, gives
\begin{align*}
	\mathrm{Var}(\xi_{s})
	\eq \E[\abs{\xi_{s}}^{2}]
	\leq \frac{2\eps^{2}}{(1-e^{-\gam})^{2}}
	\leq \frac{8\eps^{2}}{\gam^{2}} ~,
\end{align*}
where in the last inequality we have used the fact that $e^{-x} \le 1-x/2$ for $x \in [0,1]$.
Since the $\xi_{s}$'s are independent, this implies that
$
	\mathrm{Var}(W_{t}) 
	\le 16(\eps/\gam)^{2} \log_{2}{T} .
$
By Chebyshev's inequality we now obtain
$
	\Pr\lrbig{ W_{t} \ge 20 (\eps/\gam) \sqrt{\log_{2}{T}} }
	\le 1/25 ,
$
so that with probability at least $1-1/25$ we have
\begin{align*}
	W_{t} 
	~<~ \frac{20\eps}{\gam} \sqrt{\log_{2}{T}}
	\leq \frac{1}{4}
\end{align*}
for our choice of $\eps$ and $\gam$ stated in \cref{eq:eps-gam}.

Finally, recall that either $\wt\Gamma_{t}(x) = \half + W_{t}$ or $\wt\Gamma_{t}(x) = \half + W_{t} - \eps$, depending on whether $x=0$ or not.
In any case, we have $\wt\Gamma_{t}(x) \in [0,1]$ for all $x \in \set{0,1}$ with probability at least $1-1/25$, since $\eps < 1/4$.
\end{proof}

Next, we use \cref{lem:info} to show that in expectation, the regret $\wt R_{m}$ incurred on epoch $m$ grows linearly with the length of the epoch.

\begin{lemma} \label{lem:epoch-regret}
For each epoch $m$ we have $\E[\wt R_{m} \mid T_{m} = t] \ge \eps t/12$ for all $t$.
\end{lemma}

\begin{proof}
Fix some $t \in \set{0,1,\ldots,T}$, and notice that $\set{T_{m} = t} \in \F$ as the random variable $T_{m}$ is observable to the player, and in particular, is a deterministic function of $Y_{1:T}$.
Then
\begin{align*}
	\E[\wt R_{m} \mid T_{m} = t]
	\eq \E[\eps T_{m} \cdot \ind{\chi_{m} \ne 1} \mid T_{m} = t] 
	\eq \eps t \cdot \Pr(\chi_{m} \ne 1 \mid T_{m} = t) ~,
\end{align*}
and by Bayes' law, we have
\begin{align} \label{eq:bayes}
	\E[\wt R_{m} \mid T_{m} = t]
	&\eq \eps t \cdot \frac{\Pr(T_{m} = t \mid \chi_{m} \ne 1) \cdot \Pr(\chi_{m} \ne 1)}{\Pr(T_{m} = t)} \non\\
	&\geq \frac{\eps t}{2} \cdot \frac{Q_{m}(T_{m} = t)}{\Pr(T_{m} = t)} ~.
\end{align}
On the other hand, using \cref{lem:info} we obtain 
$P_{m}(T_{m} = t) \le 4e \cdot Q_{m}(T_{m} = t)$, which together with 
$
	\Pr(T_{m} = t)
	\eq \tfrac{1}{2} P_{m}(T_{m} = t) + \tfrac{1}{2} Q_{m}(T_{m} = t)
$
gives
\begin{align*}
	Q_{m}(T_{m} = t) 
	\geq \frac{ 2 }{ 1 + 4e } \cdot \Pr(T_{m} = t)
	\geq \frac{1}{6} \, \Pr(T_{m} = t) ~.
\end{align*}
Plugging this into \cref{eq:bayes} concludes the proof.
\end{proof}

\cref{thm:bandit-lower} is now a direct consequence of \cref{lem:boundness,lem:epoch-regret}.

\begin{proof}[Proof of \cref{thm:bandit-lower}]
Applying \cref{lem:epoch-regret}, we can lower-bound the expected value of $\wt R$ as
\begin{align*}
	\E[\wt R] 
	\eq \E\left[ \sum_{m=1}^{T} \wt R_{m} \right]
	\eq \E\left[ \sum_{m=1}^{T} \E[\wt R_{m} \mid T_{m}] \right]
	\geq \frac{\eps}{12} \, \E\left[ \sum_{m=1}^{T} T_{m} \right]
	\eq \frac{\eps T}{12} ~.
\end{align*}
Hence, by \cref{lem:boundness}, the expected regret of the player can be lower-bounded as
\begin{align*}
	\E[ R ] 
	\geq \lr{\frac{1}{12} - \frac{1}{25}} \cdot \eps T
	\geq \frac{\eps T}{30} ~.
\end{align*}
Using our choice of $\eps$ given in \cref{eq:eps-gam} concludes the proof.
\end{proof}

\subsection{Lower Bound for Stateful Policies}
\label{sec:policy-lower}

In this section we use the lower bound proved in \cref{sec:bandit-lower} in the hidden bandit setting to prove a similar lower bound in the stateful policies model.
Our result applies even in a very restricted case of the problem, where the reference set $\Pi$ consists of \emph{reactive} policies (see \cref{def:reactive}).
This result is stated in \cref{cor:negative}, repeated here in a more specific form.

\begin{repthm}[restated]{cor:negative}
For any randomized algorithm in the stateful policies model, there exists a set $\Pi$ of $k=3$ reactive policies over $n=3$ actions, and a sequence of oblivious reward functions $r_{1},\ldots,r_{T}$, such that expected regret of the algorithm with respect to $\Pi$ is $\Omega(T/\log^{3/2}{T})$.
\end{repthm}

\begin{proof}
Assume the contrary, namely, that there is an algorithm $A$ that achieves an expected regret of $o(T/\log^{3/2}{T})$ with respect to any set of $3$ reactive policies over $3$ actions, and for any oblivious sequence of reward functions.
We show that $A$ can be used to achieve the same expected regret in the hidden bandit model with $p=1/2$ and any oblivious assignment of rewards to the arms.
More specifically, we design an algorithm $A'$ for the hidden bandit problem based on the algorithm $A$ that obtains expected regret of $o(T/\log^{3/2}{T})$.
This would contradict \cref{thm:negative} which states that such algorithm cannot exist, proving our claim.

Consider an instance of the hidden bandit problem with $p=1/2$ and an arbitrary sequence of oblivious reward functions $r'_{1:T} : \set{0,1} \mapsto [0,1]$.
We now describe a set of reference reactive policies $\Pi = \set{\pi_{1},\pi_{2},\pi_{3}}$ and a randomized construction of reward functions $r_{1:T}$ over actions $\set{1,2,3}$, that simulates the hidden bandit instance.
For convenience, we will construct functions $r_{1:T}$ that assign reward values in the range $[-3,3]$; this only affects the constants in the resulting bounds.

\begin{description}
\item[Reference policies.]
The reference set $\Pi$ consists of $3$ reactive policies, $\pi_{1}, \pi_{2}, \pi_{3}$, that all share the same action function
\begin{align*}
	\pi : [-3,3] \mapsto \set{1,2,3} ~,
	\qquad
	\pi(r) = \big\lfloor \abs{r} \big\rfloor ~,
\end{align*}
mapping the last observed reward to the next action.
The policies only differ by their initial action on round $t=1$, where policy $\pi_{i}$ begins by playing action $i$.

\item[Reward functions.]
To construct the functions $r_{1:T}$, we draw a sequence of permutations $\sig_{1},\ldots,\sig_{T+1}$ chosen independently and uniformly at random from the set of all permutations over the elements $\set{1,2,3}$, and define for each action $i \in \set{1,2,3}$ the following reward sequence:
\begin{align} \label{eq:reduction-rewards}
	\forall ~ t \in [T]
	\qquad
	r_{t} (i) \eq 
	\text{round}\lrbig{ r'_{t}\lr{ \ind{ i \ne \sig_{t}(1) } } ,~ \sig_{t+1}(\sig_{t}^{-1}(i)) } ~,
\end{align}
where $\text{round}(r,j)$ is a randomized rounding function that rounds a reward value $r \in [-1,1]$ to $\pm j$ in a way that $\E[ \text{round}(r,j) ] = r$, namely
\begin{align*} 
	\text{round}(r,j) \eq
	\begin{cases}
		+j & \text{with probability ~ $\frac{1}{2} (1+r/j)$~,} \\
		-j & \text{with probability ~ $\frac{1}{2} (1-r/j)$~.}
	\end{cases}
\end{align*}
In particular, there are only $6$ possible reward values: $\pm 1, \pm 2, \pm 3$.

\item[Algorithm.]
Finally, we define an algorithm $A'$ for the hidden bandit problem, based on $A$.
Let $X_{1},\ldots,X_{T} \in \set{1,2,3}$ denote the sequence of actions played by $A$ on the reward functions $r_{1:T}$ and the reference set $\Pi$.
Then, on any round in which $A$ follows the function $\pi$, that is whenever $X_{t+1} = \pi(r_{t}(X_{t}))$, the algorithm $A'$ chooses \textsf{stay}; otherwise, it chooses \textsf{switch}.
The arm being pulled by $A'$ on round $t$ is given by the random variable $X'_{t} = \ind{ X_{t} \ne \sig_{t}(1) }$, and its reward on that round is $r'(X'_{t})$.
\end{description}

\begin{figure}[h]
\centering

\begin{tikzpicture}[>=latex,text height=1.5ex,text depth=0.25ex]

\tikzstyle{good} = [
	circle,
	minimum size=1cm,
	thick,
	draw=black!80,
	fill=gray!80
	]

\tikzstyle{bad} = [
	circle,
	minimum size=1cm,
	thick,
	draw=black!80,
	fill=gray!30
	]

\tikzstyle{background} = [
	rectangle,
	rounded corners=5mm,
	inner sep=0.2cm,
	fill=gray!10
	]
  
\matrix[row sep=0.35cm,column sep=1.1cm] {
    	&
	\node (t_1) {$t=1$}; &
	\node (t_2) {$t=2$}; &
	\node (t_3) {$t=3$}; &
	\node (t_4) {$t=4$}; &
	\node (t_5) {$t=5$}; &
	\node (t_6) {$\cdots$};
	\\

    	&
	\node (a_1) [bad] {$\pm 2$}; &
	\node (a_2) [bad] {$\pm 3$}; &
	\node (a_3) [good] {$\pm 1$}; &
	\node (a_4) [good] {$\pm 2$}; &
	\node (a_5) [bad] {$\pm 1$}; &
	\node (a_6)            {$\cdots$}; 
	\\
    	\node (b_0)           {$\pi^{*}$}; &
	\node (b_1) [good] {$\pm 3$}; &
	\node (b_2) [bad] {$\pm 2$}; &
	\node (b_3) [bad] {$\pm 3$}; &
	\node (b_4) [bad] {$\pm 3$}; &
	\node (b_5) [good] {$\pm 2$}; &
	\node (b_6)            {$\cdots$}; 
	\\
    	&
	\node (c_1) [bad] {$\pm 1$}; &
	\node (c_2) [good] {$\pm 1$}; &
	\node (c_3) [bad] {$\pm 2$}; &
	\node (c_4) [bad] {$\pm 1$}; &
	\node (c_5) [bad] {$\pm 3$}; &
	\node (c_6)            {$\cdots$}; 
	\\
    };
    
\path[->,very thick]
	(b_0) edge (b_1)
	(b_1) edge (c_2)
	(c_2) edge (a_3)
	(a_3) edge (a_4)
	(a_4) edge (b_5)
	(b_5) edge (b_6)
	;

\path[->,thick,dashed]
	(a_1) edge (b_2)
	(b_2) edge (b_3)
	(b_3) edge (c_4)
	(c_4) edge (a_5)
	(a_5) edge (a_6)
	;

\path[->,thick,dashed]
	(c_1) edge (a_2)
	(a_2) edge (c_3)
	(c_3) edge (b_4)
	(b_4) edge (c_5)
	(c_5) edge (c_6)
	;

\begin{pgfonlayer}{background}
	\node [
		background,
                    fit=(a_1) (b_6) (c_5),
                    ] {};
\end{pgfonlayer}

\end{tikzpicture}
\caption{An illustration of the reward functions and policies used in the reduction. The marked path represents the path of the policy $\pi^{*}$, that corresponds to the reference arm in the hidden bandit problem. The absolute value of each rounded reward on one of the paths indicates the next action on that path.} \label{fig:lb-reduction}
\end{figure}
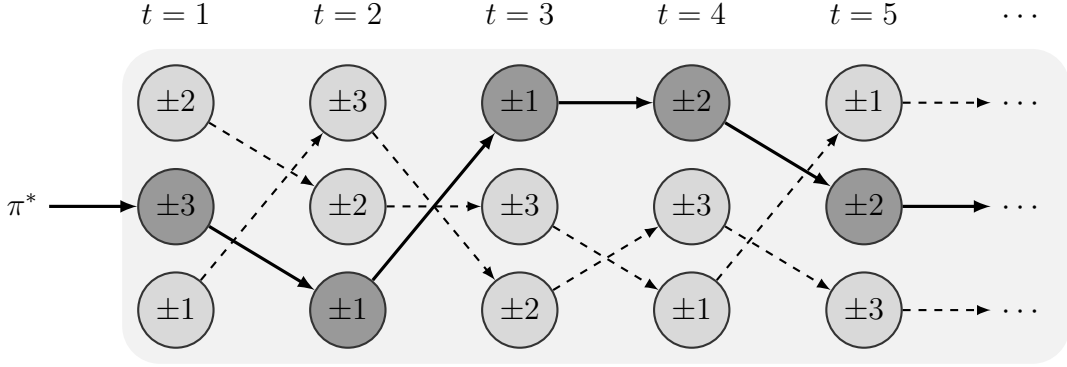

The transition function $g$ and the reward functions $r_{1:T}$ together define three disjoint random paths of actions throughout the game, each corresponding to one of the policies $\pi_{1},\pi_{2},\pi_{3}$.
Namely, for each $i=1,2,3$, the policy $\pi_{\sig_{1}(i)}$ (that plays the action $\sig_{1}(i)$ on the first round) plays the sequence of actions $\sig_{1}(i),\sig_{2}(i),\ldots,\sig_{T}(i)$ on rounds $1,2,\ldots,T$, which is disjoint from the trajectory of other policies.
This follows from the observation that, for any $i \in \set{1,2,3}$,
\begin{align} \label{eq:reduction-signal}
	\forall ~ t \in [T]
	\qquad
	\pi \lrbig{ r_{t}(\sig_{t}(i)) } \eq \sig_{t+1}(i) ~.
\end{align}
In words, if the action $\sig_{t}(i)$ was played on round $t$, then by following the  function $\pi$ the action $\sig_{t+1}(i)$ is played on round $t+1$.
The idea is that the type of rounding used for each reward value on one of the paths signals the function $\pi$ which is the next action to be played on that path.
See \cref{fig:lb-reduction} for an illustration of this structure.

The policy $\pi^{*} = \pi_{\sig_{1}(1)}$, whose action on the first round is $\sig_{1}(1)$, corresponds to the reference arm in the underlying hidden bandit problem. 
Indeed, by \cref{eq:reduction-rewards}, the expected sequence of rewards encountered along the path of $\pi^{*}$ is $r'_{1}(0),r'_{2}(0),\ldots,r'_{T}(0)$ (where the expectation is taken with respect to the randomized rounding), and thus coincides with the reward sequence of the reference arm.
The expected reward sequences corresponding to the other two paths are identical to the reward sequence of the decoy arm.

We now show that $A'$ is a valid algorithm in the hidden bandit model (with $p=1/2$).
To this end, we verify that the dynamics of the \textsf{stay} and \textsf{switch} actions are compatible with those of the hidden bandit model, namely:
\begin{enumerate}[(i)]
\item
The arm pulled by $A'$ on the first round is the reference arm with probability $1/3$ (and it is the decoy arm with probability $2/3$).
Indeed, $\Pr(X'_{1} = 0) = \Pr(X_{1} = \sig_{1}(1)) = 1/3$ since $\sig_{1}$ is chosen uniformly at random.
\item
If $A'$ chooses \textsf{stay} on round $t$, then it remains on the same arm on round $t+1$, namely $X'_{t+1} = X'_{t}$ with probability $1$.
Indeed, $A'$ chooses \textsf{stay} if $X_{t+1} = \pi(r_{t}(X_{t}))$, in which case \cref{eq:reduction-signal} implies that $X_{t+1} = \sig_{t+1}(1)$ if and only if $X_{t} = \sig_{t}(1)$, so $X'_{t+1} = X'_{t}$.
\item
If $A'$ chooses \textsf{switch} on round $t$ and $X'_{t} = 0$, then $X'_{t+1} = 1$ with probability $1$. 
This holds since $X'_{t} = 0$ means that $X_{t} = \sig_{t}(1)$, and if $A'$ chose \textsf{switch} then necessarily $X_{t+1} \ne \sig_{t+1}(1)$, which implies that $X'_{t+1} = 1$ with probability $1$.
\item
If $A'$ chooses \textsf{switch} on round $t$ and $X'_{t} = 1$, then $X'_{t+1} = 0$ with probability $1/2$.
To see this, note that $X'_{t} = 1$ implies that $X_{t} = \sig_{t}(j)$ for some $j \ne 1$, so if $A'$ chose \textsf{switch} then $X_{t+1} \ne \sig_{t+1}(j)$ and consequently $\Pr(X_{t+1} = \sig_{t+1}(1)) = 1/2$ as $\sig_{t+1}$ is a random permutation. 
This means that $X'_{t+1} = 0$ with probability $1/2$.
\end{enumerate}

Finally, notice that by \cref{eq:reduction-rewards} we have $\E[r'_{t}(X'_{t})] = \E[r_{t}(X_{t})]$ for all $t$.
Together with the fact that the total expected reward of $\pi^{*}$ equals the total expect reward of the reference arm, this implies that the expected regret of $A'$ (with respect to the reference arm) is no more than the expected regret of $A$ (with respect to $\Pi$), which is assumed to be $o(T/\log^{3/2}(T))$.
As explained earlier, this contradicts \cref{thm:negative}, and as a consequence proves our claim.
\end{proof}

\subsection{Semi-Markovian Players vs.~Consistent Adversaries}
\label{sec:semimarkov}

In this section we consider some natural restrictions on the hidden bandits setting. Some of these restrictions limit the adversary, whereas others limit the player.

We describe first two types of restrictions on the adversary, one being more severe than the other. Recall that the adversary determines the sequence of rewards for each arm. We use $r_t(0)$ ($r_t(1)$, respectively) to denote the reward of the reference arm (decoy arm, respectively) at round $t$, and assume that $0 \le r_t(i) \le 1$.

\begin{definition}[consistent adversary]
\label{def:consistent}
An adversary is {\em constant} if for every arm $i$, there is a certain value $v_i \in [0,1]$ such that $r_t(i) = v_i$ for every $1 \le t \le T$. Hence the strategy of a constant adversary is limited to selecting two values $0 \le v_1 < v_0 \le 1$. An adversary is {\em consistent} if there is a fixed offset $0 \le \Del \le 1$ such that in every round $t$, $r_t(0) - r_t(1) = \Del$. Hence the strategy of a consistent adversary is limited to selecting an offset value $0 \le \Del \le 1$ and an arbitrary sequence of rewards $r_t(0)$ for the reference arm, in the range $[\Del,1]$.
\end{definition}

When the adversary is consistent, the regret of a player is exactly $\Del$ times the number of rounds spent on the decoy arm. Every constant adversary is also a consistent adversary, with $\Del = v_0 - v_1$.

Let us now present two types of restrictions on the algorithm of the player, one being more severe than the other. Recall that a algorithm of a player determines for every round $t$ whether to {\sf switch} or {\sf stay}, depending on the history observable to the player up to and including round $t$ (namely, the sequence or rewards observed, and the time steps of all previous switch requests). It is convenient to assume a round~0 that contains an initial {\sf switch} action (there is no reward in round~0).

\begin{definition}[Markovian algorithm]
\label{def:markov}
An algorithm of the player is {\em Markovian} if there is a deterministic function $p:[0,1] \mapsto [0,1]$, which given a reward $r$ received in round $t$, maps it to a probability $p(r)$ for switching in round $t+1$. Hence, a Markovian strategy ignores all information available to the player (including the round number), except for the last reward received.
\end{definition}

\begin{definition}[semi-Markovian algorithm]
\label{def:semi-markov}
An algorithm is {\em semi-Markovian} if it depends only on the sequence of rewards obtained since the last switch request. Namely, its input is a {\em memory string} $s$ starting by a switch request, and then continuing with a nonempty sequence of rewards (observed since the last switch up to and including the current round), and its output is a probability $p(s)$ of switching. Then the player tosses a coin with bias $p(s)$. If it comes up heads the players makes a \textsf{switch} request, and consequently ``forgets" the current memory string $s$ and starts building a new memory string by placing a {\sf switch} at its beginning. If the coin comes up tails the player chooses {\sf stay}, and keeps the current memory string $s$. In either case, the reward observed in next round will be appended to the memory string.
\end{definition}

\subsubsection{Upper Bound for Consistent Adversaries}
\label{sec:consistent-upper}

We now discuss a simple algorithm for the hidden bandit problem, suitable for the case where the adversary is playing a consistent strategy. 
The algorithm, given in \cref{alg:consistent}, simply chooses \textsf{switch} with probability inversely proportional to the exponent of the last observed reward.

\begin{algorithm}
\begin{framed}
\paragraph{Algorithm vs.~consistent adversaries} (parameters: $\eta$)
\begin{itemize}[leftmargin=\algindent]
\item
For $t=1,\ldots,T$:
observe reward $r_t$ and \textsf{switch} with probability $\frac{1}{2} e^{-\eta r_t}$ (otherwise \textsf{stay})
\end{itemize}
\vspace{-2ex}
\end{framed}
\caption{An algorithm for the hidden bandit problem with a consistent adversary, that keeps a constant gap between the rewards sequences of the two arms.}
\label{alg:consistent}
\end{algorithm}

The intuition behind this algorithm is straightforward: since one arm is constantly better than the other, the player is more likely to switch when he is on the worse arm.
We prove that this algorithm performs better than our algorithm for the general case, and achieves an expected regret of $O(T/\log{T})$ against any consistent adversary.

\begin{theorem} \label{thm:consistent}
For $\eta = \frac{1}{2} \log{T}$, the expected regret of \cref{alg:consistent} is
$
	O\big( \frac{ T } { p \log T } \big)
$
against any consistent adversary. 
\end{theorem}


\begin{proof}[Proof of \cref{thm:consistent}]
Let $X_{t}$ denote the arm assigned to the algorithm on round $t$.
Denote by $q^{t}_{i}$ the probability that the algorithms requests a \textsf{switch} on round $t$ given that it is pulling arm $i$ ($i=0,1$).
Then 
\begin{align*}
	\Pr(X_{t+1}=1 \mid X_{t}=0) &\eq q^{t}_{0} ~, \\
	\Pr(X_{t+1}=0 \mid X_{t}=1) &\eq p q^{t}_{1} ~.
\end{align*}
Hence, we can view the sequence $X_{1},X_{2},\ldots,X_{T}$ as a trajectory of a two-state Markov chain, with its transition kernel on time $t$ given by
\begin{align*}
	Q_{t} \eq 
	\left(\begin{array}{cccc}
		1-q^{t}_{0} & q^{t}_{0} \\
		p q^{t}_{1} & 1-p q^{t}_{1}
	\end{array}\right) ~.
\end{align*}
Our goal is to prove that this chain mixes quickly to a steady-state distribution.
Note, however, that the chain is \emph{time inhomogeneous} (there is a different transition matrix on each round), so standard bounds on mixing times do not apply.
Nevertheless, our analysis hinges on the fact that there exists a single distribution which is stationary with respect to all transition kernels simultaneously, and shows that the chain mixes to that distribution.

First, we identify the steady-state distribution shared by all transition kernels.

\begin{lemma} \label{lem:station}
The distribution
\begin{align} \label{eq:pi}
	\mu \eq
	\lr{ \frac{p}{p+e^{-\eta \Del}} 
		~,~ \frac{e^{-\eta \Del}}{p+e^{-\eta \Del}} }
\end{align}
is stationary with respect to all kernels $P_{t}$.
That is, it holds that $\mu P_{t} = \mu$ for all $t$.
\end{lemma}

Next, we prove that the inhomogeneous chain mixes to the common stationary distribution $\mu$.
In the following, we let $\mu^{t}$ denote the distribution of $X_{t}$, namely, the probability distribution over the arms induced by the algorithm on round $t$.

\begin{lemma} \label{lem:mixing}
For all $t$, we have $\norm{\mu^{t} - \mu}_{1} \le 2 (1 - p e^{-\eta})^{t-1}$.
\end{lemma}

%
%


The proofs of both lemmas are deferred to the end of the section.
Now, note that the expected regret of the player can be written in terms of the distributions $\mu^{t}$ as
\begin{align*}
	\E[R_{T}] \eq
	\Del \cdot \sum_{t=1}^{T} \mu^{t}_{1} ~.
\end{align*}
Suppose that the player could sample directly from the stationary distribution $\mu$ (on each round independently).
Then, his expected regret would be
\begin{align*}
	\Del \cdot \sum_{t=1}^{T} \mu_{1}
	\eq	T \cdot \frac{ \Del e^{-\eta \Del} }{ p + e^{-\eta \Del} }
	\leq	\frac{T}{\eta p} \cdot \eta\Del e^{-\eta \Del} 
	\leq 	\frac{T}{2\eta p} ~,
\end{align*}
where in the last inequality we have used the fact that the function $x \mapsto x e^{-x}$ for $x \ge 0$ is maximized at $x=1$.
Using \cref{lem:mixing}, we can bound the difference between this regret and the player's actual regret:
\begin{align*}
	\Del \cdot \sum_{t=1}^{T} (\mu^{t}_{1} - \mu_{1})
	\leq	\sum_{t=1}^{T} \norm{ \mu^{t} - \mu }_{1}
	\leq 	\sum_{t=1}^{\infty} 2 (1 - p e^{-\eta})^{t-1}
	\eq 	\frac{2e^{\eta}}{p} ~.
\end{align*}
Overall, we have
\begin{align*}
	\E[R_{T}] 
	\leq 	\frac{ 2 e^{\eta} }{ p } + \frac{ T } { 2\eta p } ~.
\end{align*}
Choosing $\eta = \frac{1}{2} \log T$ we obtain
\begin{align*}
	\E[R_{T}] 
	\leq 	\frac{1}{p} \lr{ 2\sqrt{T} + \frac{T}{\log T} }
	\eq 	O\lr{ \frac{ T } { p\log T } } ~,
\end{align*}
as claimed.
\end{proof}

Finally, we provide the proofs of the lemmas used in our analysis above.

\begin{proof}[Proof of \cref{lem:station}]
A simple calculation verifies that the distribution
\begin{align*}
	\nu_{t} \eq 
	\lr{ \frac{p q^{t}_{1}}{q^{t}_{0}+p q^{t}_{1}} 
		~,~ \frac{q^{t}_{0}}{q^{t}_{0}+p q^{t}_{1}} }
\end{align*}
is stationary with respect to $P_{t}$.
However, observe that
\begin{align*}
	\frac{q^{t}_{0}}{q^{t}_{1}}
	\eq \frac{ e^{-\eta r_{t}(0)} }{ e^{-\eta r_{t}(1)} }
	\eq e^{-\eta (r_{t}(0) - r_{t}(1))}
	\eq e^{-\eta \Del} ~,
\end{align*}
which gives
\begin{align*}
	\nu_{t} 
	\eq \lr{ \frac{p}{p+ q^{t}_{0}/q^{t}_{1}} 
		~,~ \frac{q^{t}_{0}/q^{t}_{1}}{p+q^{t}_{0}/q^{t}_{1}} }
	\eq \lr{ \frac{p}{p+e^{-\eta \Del}} 
		~,~ \frac{e^{-\eta \Del}}{p+e^{-\eta \Del}} }
	\eq \mu ~.
\end{align*}
for all $t$. 
\end{proof}

For the proof of \cref{lem:mixing} we need the following technical result, which is a variant of Theorem~4.9 in \cite{levin2009markov}.

\begin{lemma} \label{lem:Q}
Let $Q$ be a $k \times k$ stochastic matrix such that $Q_{ij} \ge \eps$ for some $\eps > 0$ and all $i,j$.
Then for any two distribution vectors $\mu$ and $\nu$ we have
$$
	\norm{\mu Q - \nu Q}_{1} 
	\leq (1-k\eps) \cdot \norm{\mu-\nu}_{1} ~.
$$
\end{lemma}

\begin{proof}
We can write $Q = (1-k\eps) M + \eps J$ where $Q$ is a stochastic matrix and $J$ denotes the all-ones $k \times k$ matrix.
Now, for any two distributions $\mu,\nu$ we have $J\mu = J\nu = \mathbf{1}$.
Consequently,
\begin{align*}
	\norm{\mu Q - \nu Q}_{1}
	\eq (1-k\eps) \cdot \norm{(\mu - \nu)M}_{1} ~.
\end{align*}
It remains to bound the norm on the right-hand size. 
We have
\begin{align*}
	\norm{(\mu-\nu)M}_{1}
	&\eq 	\sum_{j=1}^{k} \Big| \sum_{i=1}^{k} (\mu_{i}-\nu_{i}) M_{ij} \Big| \\
	&\leq 	\sum_{i=1}^{k} \sum_{j=1}^{k} |\mu_{i}-\nu_{i}| M_{ij}
	\eq 	\sum_{i=1}^{k} |\mu_{i}-\nu_{i}| \sum_{j=1}^{k} M_{ij}
	\eq 	\sum_{i=1}^{k} |\mu_{i}-\nu_{i}| \\
	&\eq 	\norm{\mu-\nu}_{1} ~,
\end{align*}
which completes the proof.
\end{proof}

\begin{proof}[Proof of \cref{lem:mixing}]
Since $q^{t}_{0}, q^{t}_{1} \le \frac{1}{2}$, the smallest entry in the matrix $Q_{t}$ is $p q^{t}_{1}$ which is bounded from below by
$$
	p q^{t}_{1} 
	\eq 	\tfrac{1}{2} p e^{-\eta r_{t}(1)}
	\geq	\tfrac{1}{2} p e^{-\eta}
$$
as the reward $r_{t}(1)$ is at most one.
Hence, \cref{lem:Q} above implies that for all $t$,
\begin{align*}
	\norm{\mu^{t+1} - \mu}_{1}
	\eq 	\norm{\mu^{t} Q_{t} - \mu Q_{t}}_{1} 
	\leq 	(1 - p e^{-\eta}) \cdot \norm{\mu^{t}-\mu}_{1} ~,
\end{align*}
where we have also used the stationarity of $\pi$.
By repeating this argument, we get
\begin{align*}
	\norm{\mu^{t+1} - \mu}_{1} 
	\leq 	(1 - p e^{-\eta})^{t} \cdot \norm{\mu^{1} - \mu}_{1} 
	\leq	2 (1 - p e^{-\eta})^{t} ~,
\end{align*}
since the $L_{1}$-distance between two distributions is at most $2$.
\end{proof}

\subsubsection{Lower Bound for Semi-Markovian Algorithms}
\label{sec:consistent-lower}

We refer to the following constant strategy for the adversary as the {\em \textsf{MT} strategy}. \textsf{MT} stands for {\em monotonicity testing}, as this strategy (and the first part of the analysis of Case~2 in the proof of \cref{thm:MT} below) is a variation on a procedure of \cite{raskhodnikova1999monotonicity} for testing whether a one variable function is monotone.  For simplicity and with only negligible affect on the end results, assume that $1+ \log T$ is a power of~2.

\paragraph{The constant strategy \textsf{MT}.}
The adversary chooses at random two integer values $1 \le k_1 < k_0 \le \log T$, subject to the following constraint. Let $r$ be the largest power of~2 that divides either $k_1$ or $k_0$ (whichever gives a larger value for $r$). Then the constraint is that $k_0 - k_1 \le 2^r$. For $0 \le r \le \log \log T$, say that a pair $k_1 < k_0$ is in {\em class} $r$ if $2^{r-1} < k_0 - k_1 \le 2^r$. Observe that for each value of $r$ there are at most $\log T$ pairs $(k_1,k_0)$ in class $r$. Let $c = \sum_{r=0}^{\log\log T} \frac{1}{(r+1)^2}$ and observe that $c < 2$ (regardless of $T$).
The probability distribution from which the pair $(k_1,k_0)$ is chosen is as follows: first an integer value $0 \le r \le \log\log T$ is chosen with probability $\frac{1}{c(r+1)^2}$. Then a pair $(k_1,k_0)$ from class $r$ is chosen uniformly at random. Finally, given the chosen $k_1$ and $k_0$, the adversary sets $v_0 = \frac{k_0}{\log T}$ and $v_1 = \frac{k_1}{\log T}$.

\begin{theorem}
\label{thm:MT}
In the hidden bandit setting, regardless of the value of the parameter $p$, if the player is restricted to use semi-Markovian strategies, then his expected regret against the \textsf{MT} strategy is $\Omega( T/\log T )$.
\end{theorem}

\begin{proof}
To obtain a lower bound on the regret it suffices to consider deterministic strategies for the player, because the strategy of the adversary is already fixed to be \textsf{MT}. Given that each arm has constant reward under the \textsf{MT} strategy, then a deterministic semi-Markovian strategy of the player is simply a function $g:[0,1] \mapsto \{1, 2, \ldots, T\}$ that maps the observed reward $r$ to how many rounds should elapse since the previous switch until the next switch.
Consider an arbitrary such strategy $g$ of the player, and consider the function $f(k) = g(k/\log T)$, defined on integer $k$ in the range $[1,\log T]$. Let $\mathsf{LIS}(f)$ be the length of the longest monotone increasing subsequence of the sequence $f(1), f(2), \ldots, f(\log T)$. We consider two cases.

\paragraph{\bf Case 1: $\mathsf{LIS}(f) \ge \frac{4}{5}\log T$.}

Let $i_1, \ldots, i_{\ell}$ be the indices of an increasing subsequence of length $\ell \ge \frac{4}{5}\log T$. For at most $\frac{1}{2}\log T$ values of $j$ we have that $f(i_{j+1}) \ge 4f(i_j)$, as otherwise $f(i_{\ell}) > 4^{\log T / 2}f(i_1) = Tf(i_1) \ge T$, which is outside the range of the function $p$. Observe also that the increasing subsequence, being of length at least $\frac{4}{5}\log T$, must contain at least $\frac{3}{5}\log T$ consecutive pairs, where a consecutive pair is a value $j$ such that $i_{j+1} = 1 + i_j$. Hence at least $\frac{1}{10}\log T$ consecutive pairs differ by a factor less than~4. Namely, there are at least $\frac{1}{10}\log T$ values of $i$ for which $f(i+1) \le 4f(i)$. Refer to such a pair $(i,i+1)$ as a {\em dangerous pair}. The dangerous pair is in class~0 (with classes as defined in the \textsf{MT} strategy). The probability of the \textsf{MT} strategy to choose class~0 is $\frac{1}{c} \ge \frac{1}{2}$, and if chosen, the dangerous pair is chosen with probability $\frac{1}{\log T}$. On the dangerous pair, the expected number of rounds the player spend on the decoy arm (the one with reward $\frac{i}{\log T}$ rather than $\frac{i+1}{\log T}$) is at least $\frac{T}{5}$, because $f(i+1) \le 4f(i)$. On the decoy arm the regret is $\frac{1}{\log T}$. Hence the expected regret is at least:
$$
	\frac{1}{10}\log T \cdot \frac{1}{2}
		\cdot \frac{1}{\log T} \cdot  \frac{T}{5} \cdot \frac{1}{\log T}
	\eq \frac{T}{100\log T} ~.
$$

\paragraph{Case 2: $\mathsf{LIS}(f) < \frac{4}{5}\log T$.}

Consider a graph with vertices labeled from $1$ to  $\log T$, where an edge connects vertex $i$ to vertex $j > i$ if and only if they form a {\em decreasing pair} with respect to $f$, namely $f(j) < f(i)$.  Consider a maximal matching in this graph. Its size is at least $\frac{1}{10} \log T$, because otherwise all vertices not involved in a maximal matching would form an increasing subsequence longer than $\frac{4}{5} \log T$. Consider an arbitrary matching edge $(i,j)$, and let $h$ in the range $i \le h \le j$ be such that the power of two that divides it is highest. If $h \in \{i,j\}$ then the pair $(i,j)$ is a possible choice of the \textsf{MT} algorithm. We call this an \textsf{MT}-pair. If $i < h < j$ then both $(i,h)$ and $(h,j)$ are \textsf{MT}-pairs, and at least one of them is decreasing. Hence each matching edge contributes at least one decreasing \textsf{MT}-pair. A simple case analysis (that is omitted) shows that two different matching edges cannot possibly contribute the same decreasing \textsf{MT}-pair. Hence there are at least $\frac{1}{10} \log T$ decreasing \textsf{MT}-pairs. Observe that these pairs need not be disjoint: the same $h$ vertex can participate in several (or even all) pairs.

Let $c$ be the constant in the definition of the \textsf{MT} strategy. Then there must be a choice of integer $r$ in the range $0 \le r \le \log\log T$ that is {\em dangerous} in the sense that there are $\frac{\log T}{10c(r+1)^2}$ decreasing \textsf{MT}-pairs of class~$r$. This dangerous $r$ is chosen with probability $\frac{1}{c(r+1)^2}$. Given that a dangerous $r$ is chosen, the probability of choosing a decreasing \textsf{MT}-pair of class $r$ is at least $\frac{1}{10c(r+1)^2}$, because there are at most $\log T$ pairs in a class. If a decreasing \textsf{MT}-pair is selected, the player spends in expectation at least $T/2$ steps on the decoy arm. In these steps, the regret of the player is at least $\frac{2^{r-1}}{\log T}$.
Hence the expected regret is at least (using also $c < 2$):
$$
	\frac{1}{c(r+1)^2} \cdot \frac{1}{10c(r+1)^2}
		\cdot \frac{T}{2} \cdot \frac{2^{r-1}}{\log T}
	\geq \frac{2^{r+1}T}{320(r+1)^4\log T} ~.
$$
The expression $\frac{2^{r+1}}{(r+1)^4}$ is minimized when $r = 5$, giving roughly $\frac{1}{20}$. Hence the expected regret in case~2 is at most $\frac{T}{7000\log T}$.

We have not attempted to optimize the leading constant in the $\Omega$ notation, and no doubt it can be substantially improved beyond the bounds shown in our proof.
\end{proof}

The combination of \cref{thm:consistent,thm:MT} implies the following result.

\begin{repthm}[restated]{thm:semimarkov}
Consider a hidden bandit problem with $p=\frac{1}{2}$, where the adversary is restricted to be consistent and the player is restricted to use semi-Markovian algorithms.
Then, the player can guarantee expected regret $O( T/\log T )$, and this is the best possible guarantee in this setting.
Moreover, to achieve this guarantee Markovian strategies suffice, and to block stronger guarantees constant adversaries suffice.
\end{repthm}

\subsection{Reactive Adversaries}
\label{sec:reactive}

In this section we discuss an extension of our basic setting, where the adversary is allowed to be \emph{reactive}.
Reactive adversaries are adversaries that set the reward given at any round as a function of the actions of the player on that and the $\ell$ previous rounds (for some fixed $\ell$).
In other words, a reactive adversary determines a sequence of reward functions $r_{1},\ldots,r_{T} : [n]^{\ell+1} \mapsto [0,1]$ before the game begins, where each function in the sequence is used to map $\ell+1$ consecutive actions of the player to a reward value.
With the notation of \cref{sec:model}, the expected regret of the player in this case is given by
\begin{align*}
	\mathrm{Regret}_{T} \eq
	\max_{\pi \in \Pi} \sum_{t=1}^{T} r_{t}(x_{t}^{\pi},\ldots,x_{t-\ell}^{\pi})
	- \E\left[ \sum_{t=1}^{T} r_{t}( X_{t},\ldots,X_{t-\ell} ) \right] ~.
\end{align*}
For simplicity, we focus on the case where $\ell = 1$.

When the reference policies are stateless, previous work show that the player can obtain $\tO(T^{2/3})$ expected regret against a reactive adversary \citep{arora2012online}, and this rate is best possible in general \citep{dekel2013bandits}. 
These results do not hold when the reference policies are stateful; indeed, our lower bound (\cref{cor:negative}) clearly extends to reactive adversaries as the oblivious adversary we have considered above is a special case of a reactive adversary.


Our algorithms can be adapted to the reactive adversary setting.
We sketch our approach, while omitting the technical details.
Essentially, the only issue we have to address is providing a new implementation of a \textsf{switch} action in \cref{alg:policies}, that switches at some round $t$ to a random policy and initializes it in a random state.
The difficulty is that even if the switch was successful and resulted in the best policy at its correct state for round $t$, the reward observed as feedback on round $t$ is also a function of the action in round $t-1$, where a different policy was followed (and a possibly different feedback was observed). 
Hence applying the state transition function with this reward might cause the policy to reach an incorrect state at round $t+1$.
We deal with this problem by spending two rounds, $t$ and $t+1$, on implementing the switch operation. In round $t$ the player plays a random action. In round $t+1$ the player picks a random policy to switch to and guesses its internal state (as done by \cref{alg:policies}), and plays the action recommended by the guessed policy in the guessed state. 
With probability $1/(knS)$ all the following conditions hold simultaneously: the player guessed the best reference policy at its correct state for round $t+1$, and the actions performed in rounds $t$ and $t+1$ are exactly as would have been chosen had the player followed the best policy all along. 
Thereafter, the feedback received in round $t+1$ puts the player on the right track of the best policy.
This approach retains the sublinear regret rate provided by \cref{alg:policies} (in terms of $T$), but the dependence on the constants involves also $n$, and not only $k$ and $s$.

\bibliographystyle{abbrvnat}
\bibliography{bib}

\appendix

\section{Martingales and Locally-Repetitive Strings}
\label{sec:doob}

\cref{lem:localrepetitive} is central to our work. 
Hence, it is instructive to see another proof for it. 
Let us recall for this purpose Doob's upcrossing inequality for martingales.
Let $X_1, \ldots, X_n$ be a martingale, namely a sequence of random variables such that $X_i = E[X_{i+1} \mid X_{1},\ldots,X_{i}]$ for all $i$, and suppose that the range of values of the martingale in bounded in the sense that $0 \le X_n \le 1$ (hence the same necessarily holds for all $X_i$). Fixing $0 \le a < b \le 1$, an $(a,b)$-upcrossing in a sequence are two indices $i< j$ such that $X_i \le a$ and $X_j \ge b$. The number of $(a,b)$-upcrossings in the sequence is the largest number $t$ such that there are indices $i_1 < j_i < i_2 < j_2 \ldots < i_t < j_t$, and for every $1 \le \ell \le t$, there is an $(a,b)$ upcrossing in $(i_{\ell}, j_{\ell})$. 
The following lemma is known as Doob's upcrossing inequality for martingales, for which we sketch a simple proof for completeness.

\begin{lemma}[Doob's upcrossing inequality]
\label{lem:doob}
For the setting as above, the expected number of $(a,b)$ upcrossing is at most $a/(b-a)$.
\end{lemma}

\begin{proof}[Proof (sketch)]
Think of $X_i$ as the value of a stock at time $i$. 
Due to the martingale property, there is no trading policy for the stock that gains money in expectation. Consider the strategy of buying the stock as soon as it drops below $a$, and selling it as soon as it moves above $b$, and so on. At the last round the stock is sold regardless of its value. If the number of crossings is $U$, this strategy makes a profit of $(b-a) U$. Selling on the last round loses at most $a$ (the maximum possible value at which the stock was bought). Hence $(b-a)\E[U] - a \le 0$, proving the lemma.
\end{proof}

Fix $\epsilon > 0$ that divides~1, and for integer $0 \le m < 1/\epsilon$, we refer to an $(m\epsilon, (m+1)\epsilon)$ upcrossing as an $\epsilon$-upcrossing. Summing over all $m$, \cref{lem:doob} implies that the expected number of $\epsilon$-upcrossings is 
$$
	\frac{1}{\epsilon}\sum_{m=1}^{1/\epsilon} (m-1)\epsilon 
	~<~ \frac{1}{2\epsilon^2} ~.
$$
A similar bound applies by symmetry to the analogous notion of $\epsilon$-downcrossing. Hence altogether the expected number of $\epsilon$ crossings is at most $1/\eps^{2}$.

We can now sketch an alternative proof for \cref{lem:localrepetitive} (with somewhat weaker bounds). 

\begin{proof}[Proof of \cref{lem:localrepetitive} (sketch)]
Starting from $s$, consider the process of partitioning $s$ into $d$ substrings and choosing one of them at random, and doing so recursively until a single character $u$ is reached. The averages $x_s, \ldots, x_u$ encountered on such a random path form a martingale sequence, with final value in $[0,1]$. In expectation, at most $4/\eps^{2}$ of the steps where an $\epsilon/2$-crossing. Observe that for every string $v$ that is encountered along the way, if $v$ is not $(d,\epsilon)$-repetitive, then there is probability at least $1/d$ of moving to a substring $u$ that inflicts an $\epsilon/2$ crossing (the inequality $|x_v - x_u| > \epsilon$ implies that within this range there is a crossing of width $\epsilon/2$ aligned at a multiple of $\epsilon/2$). Hence at most $4d/\eps^2$ steps went through strings that are not $(d,\epsilon)$-repetitive. As there are $k$ steps, this implies that $\delta \le 4d/k\eps^{2}$.
\end{proof}

\end{document}